\pdfoutput=1

\documentclass[11pt]{article}

\usepackage[final]{acl}

\usepackage{times}
\usepackage{latexsym}

\usepackage[T1]{fontenc}

\usepackage[utf8]{inputenc}

\usepackage{microtype}

\usepackage{inconsolata}

\usepackage{graphicx}
\usepackage{amsmath}
\usepackage{mathtools}
\usepackage{bbm}
\usepackage{amssymb}
\usepackage{algorithm}          
\usepackage{overpic}
\usepackage{algpseudocode}
\usepackage{subcaption}
\usepackage{multirow}
\usepackage{subcaption}
\usepackage{booktabs}
\usepackage{tikz}
\usepackage{tabularx}
\usepackage{bibunits}

\newtheorem{theorem}{Theorem}

\title{
\framework: Few-step Match Many-step Diffusion Language Model on Sequence-to-Sequence Generation--Full Version}

\author{
 \textbf{Dat Nguyen-Cong\textsuperscript{1}},
 \textbf{Tung Kieu\textsuperscript{2}},
 \textbf{Hoang Thanh-Tung\textsuperscript{3}}
\\
\\
\textsuperscript{1}FPT Software AI Center, FPT Corporation,
\\
\textsuperscript{2}Department of Computer Science, Aalborg University, Denmark, 
\\
\textsuperscript{3}Quantum AI and Cyber Security Institute, FPT Corporation
\\
 \small{
   \href{mailto:dat27072002@gmail.com}{dat27072002@gmail.com},
   \href{mailto:tungkvt@cs.aau.dk}{tungkvt@cs.aau.dk},
   \href{mailto:htt210@gmail.com}{htt210@gmail.com}
 }
}

\usepackage[colorinlistoftodos,prependcaption,textsize=tiny]{todonotes}

\usepackage{xspace}

\newcommand{\framework}{\texttt{FastDiSS}\xspace}
\newcommand{\scheduler}{\texttt{SCP}\xspace}
\newcommand{\noisescaling}{\texttt{MANS}\xspace}

\begin{document}

\maketitle
\setcounter{page}{1}

\begin{abstract}
Self-conditioning has been central to the success of continuous diffusion language models, as it allows models to correct previous errors. Yet its ability degrades precisely in the regime where diffusion is most attractive for deployment: few-step sampling for fast inference. 
In this study, we show that when models only have a few denoising steps, inaccurate self-conditioning induces a substantial approximation gap; this mistake compounds across denoising steps and ultimately dominate the sample quality.
To address this, we propose a novel training framework that handles these errors during learning by perturbing the self-conditioning signal to match inference noise, improving robustness to prior estimation errors. In addition, we introduce a token-level noise-awareness mechanism that prevents training from saturation, hence improving optimization. Extensive experiments across conditional generation benchmarks demonstrate that our framework surpasses standard continuous diffusion models while providing up to 400x faster inference speed, and remains competitive against other one-step diffusion frameworks.
\end{abstract}
\section{Introduction}
\label{sec:introduction}
Diffusion models have recently emerged as a compelling alternative to autoregressive text generation, matching and in some settings surpassing autoregressive models in quality and controllability~\cite{DBLP:journals/corr/abs-2502-09992,DBLP:journals/corr/abs-2507-15857}. 
A key advantage of diffusion models lies in their ability to generate all tokens in parallel, offering a linear-time decoding rather than a quadratic complexity as in autoregressive decoding ~\cite{radford2018improving,radford2019language,DBLP:conf/nips/BrownMRSKDNSSAA20}. This property makes diffusion models attractive for a broad range of natural language processing tasks beyond unconditional generation, such as conditional and structured sequence modeling~\cite{DBLP:conf/iclr/GongLF0K23,DBLP:conf/nips/YeGCZGSWJLBK24,DBLP:journals/corr/abs-2508-15487}.

Despite these advantages, diffusion language models face a central efficiency bottleneck: high-quality generation typically requires a long reverse process with many denoising steps~\cite{DBLP:conf/iclr/GongLF0K23}. The resulting iterative sampling cost eventually offset the latency gains from parallel token generation. To mitigate this issue, a widely adopted trick is to use self-conditioning~\cite{DBLP:conf/iclr/ChenZH23,DBLP:conf/nips/GulrajaniH23}, which reuses previous predictions as an additional conditioning signal to improve prediction under fewer steps. 
While self-conditioning indeed strengthens few-step sampling, we find that it introduces underappreciated failure modes that become critical precisely in the fast-inference application.

\paragraph{Mismatching between training-inference self-condition under few-step denoising.}
Self-conditioning introduces an intrinsic mismatch between training and inference~\cite{DBLP:conf/emnlp/Schmidt19,DBLP:conf/iclr/NingLSSE24,DBLP:conf/naacl/GaoG0ZZ0X24}. 
During training, a model can be conditioned on ground-truth targets, whereas at inference, it must be conditioned on its own imperfect previous predictions. This distribution shift induces error accumulation along the reverse trajectory, leading to sampling drift~\cite{DBLP:conf/nips/DarasDDD23} and degraded generation quality. 
Our analysis shows that the problem is amplified in few-step settings: predictions made at high noise levels differ significantly from those made later at low noise, turning the reused self-conditioning signal into a biased condition. Consequently, self-conditioning can become unstable, and in the worst case, the reused estimates can steer subsequent denoising updates in the wrong direction.

\paragraph{Loss saturation in the late-stage training.} 
Diffusion language models often fit the denoising objective quickly early in training, but subsequently exhibit a pronounced loss plateau. 
This slow improvement suggests that the sampled noise levels become insufficiently informative: the training signal is dominated by ``easy'' cases where tokens are already predicted with high confidence, leading to inefficient learning. 
Prior works attribute this behavior in part to applying a uniform noise schedule across tokens, which ignores token-wise heterogeneity in denoising difficulty~\cite{DBLP:conf/naacl/YuanYTHH24}.
Consistent with this view, our analysis shows that uniform noise sampling is suboptimal. In particular, increasing noise for well-predicted tokens yields a more effective learning signal and improves optimization efficiency, thus achieving lower evaluation loss than models trained with the uniform schedule.

To address these challenges, we propose \textbf{Fast Diffusion Sequence-to-Sequence} (\framework), a novel training framework designed to improve the robustness and efficiency of the diffusion model in the few-step setting. Building on self-conditioning, \framework introduces two complementary components that directly target the above failure modes. 
First, we propose \textbf{Self-conditioning Perturbation} (\scheduler), a simple regularization strategy for self-conditioning.
Initially, during training, we obtain this condition by running the denoising network on a more-noised forward process, producing a weaker and noisier estimation.
We then train the network conditioned on this corrupted signal, better matching inference-time errors and reducing sampling drift.
Second, we introduce \textbf{Model-aware Noise Scaling} (\noisescaling), a token-level noise allocation strategy that dynamically adjusts noise based on denoising confidence.
\noisescaling applies higher noise to high-confidence tokens, prevents trivial training, and further reduces self-conditioning errors at high noise levels.

We evaluate \framework on six benchmarks covering diverse sequence-to-sequence tasks and few-step generation settings. Across settings, \framework consistently narrows the gap between few-step and many-step sampling, outperforming prior text diffusion baselines in both quality and efficiency. In particular, \framework improves \textit{BLEU} score~\cite{DBLP:conf/acl/PapineniRWZ02} while achieving substantial speedups ranging from 4$\times$ to 400$\times$, and remains competitive with other few-step diffusion approaches.

In summary, our contributions are threefold:
(1) we identify and analyze two bottlenecks that limit self-conditioned diffusion language models under few-step samplings, highlighting the roles of discretization-induced mismatch and late-stage training saturation; (2) we introduce \framework, combining \scheduler to regularize self-conditioning under realistic inference noise and \noisescaling to avoid trivial denoising via confidence-driven token-wise noise allocation; and (3) we demonstrate consistent gains on six benchmarks, showing that \framework improves generation quality while substantially accelerating inference.

\section{Background}
\label{sec:background}
\subsection{Denoising Diffusion Probabilistic Models} 
We revisit Gaussian diffusion process in its continuous-time formulation~\cite{DBLP:conf/iclr/0011SKKEP21,DBLP:conf/nips/ChuangHLGLCL24}, which defines a trajectory $\{\boldsymbol{z}_{t}\}_{t=0}^{1}, t \in \mathbb{R}$ of increasing noise starting from the clean data $\boldsymbol{z}_{0}\sim p(\boldsymbol{z}_{0})$ and ending with $\boldsymbol{z}_{1}\sim \mathcal{N}(0,\mathbf{I})$.
For any $t$, the noise schedule comprises the decay factor $\alpha_{t}$ and the diffusion rate $\sigma_{t}$, which are strictly positive and monotonic over time.

Given that $q(\boldsymbol{z}_{t}|\boldsymbol{z}_{0})$ satisfies the Markovian property, the forward process is formulated as follows.
\begin{align}
    q(\boldsymbol{z}_{t}|\boldsymbol{z}_{0})&=\mathcal{N}(\boldsymbol{z}_{t}; \alpha_{t} \boldsymbol{z}_{0},\sigma_{t}^{2}\mathbf{I}), \label{eq:forward_1}\\ q(\boldsymbol{z}_{t}|\boldsymbol{z}_{s})&=\mathcal{N}(\boldsymbol{z}_{t}; (\alpha_{t}/\alpha_{s})\boldsymbol{z}_{s},\sigma_{t|s}^2\mathbf{I})\label{eq:forward_2}
\end{align}
\noindent Here, $0 \leq s < t \leq 1$ and $\sigma_{t|s}^2=\sigma_{t}^{2}-(\alpha_{t}^{2}/\alpha_{s}^{2})\sigma_{s}^{2}$.
As $t \to 1$, $\alpha_{t} \to 0$ and $\sigma_{t} \to 1$, the endpoint follows a Gaussian distribution.

The goal of the diffusion model is to denoise $\boldsymbol{z}_0\sim p(\boldsymbol{z}_0|\boldsymbol{z}_t)$ through a neural network $D_\theta(\boldsymbol{z}_t)$, which is trained using a mean-squared error loss:
\begin{equation}
    \mathcal{L}_{\text{diffusion}} = \mathbb{E}_{\boldsymbol{z}_{0},t \sim \mathcal{U}[0, 1]} [||D_\theta(\boldsymbol{z}_t)-\boldsymbol{z}_{0}||^2] 
\end{equation}
\noindent Here, $\mathcal{U}[0, 1]$ denotes the continuous uniform distribution. 
Recursive sampling from the distribution 
\begin{align}
    p(\boldsymbol{z}_s|\boldsymbol{z}_t) &=                          \mathbb{E}_{p(\boldsymbol{z}_0|\boldsymbol{z}_t)}[q(\boldsymbol{z}_{s}|\boldsymbol{z}_t,\boldsymbol{z}_0)] \nonumber\\ &\approx\mathbb{E}_{p(\hat{\boldsymbol{z}}_\theta^t|\boldsymbol{z}_t)}[q(\boldsymbol{z}_{s}|\boldsymbol{z}_t,\hat{\boldsymbol{z}}_\theta^t)],
    \label{eq:posterior}
\end{align}
where $\hat{\boldsymbol{z}}_\theta^t=D_\theta(\boldsymbol{z}_t)$, starting at $\boldsymbol{z}_1\sim \mathcal{N}(0,\mathbf{I})$, enables generating data from $p(\boldsymbol{z}_{0})$. 
Full expression of $q(\boldsymbol{z}_{s}|\boldsymbol{z}_t,\boldsymbol{z}_0)$ is demonstrated in Appx.~\ref{subsec:posterior_dist_derv}.

\begin{figure*}[t!]
    \centering
    \includegraphics[width=0.9\linewidth]{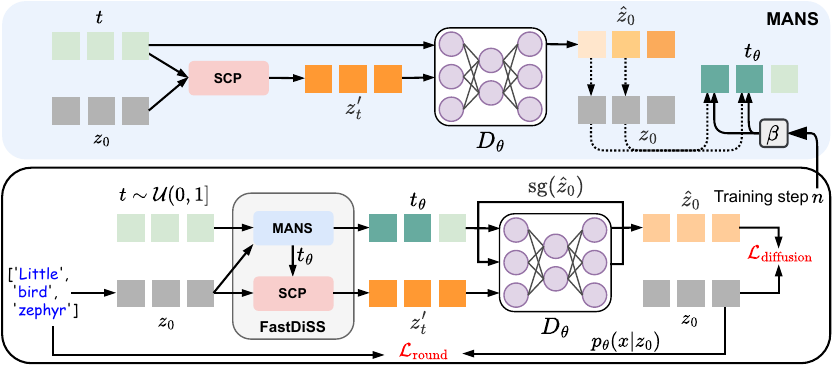}
    \caption{Overview of \framework. The tokenized sequence is first encoded to $\boldsymbol{z}_0$, while concurrently sampling the initial timestep $t$. Both $\boldsymbol{z}_0$ and $t$ are passed into \noisescaling to obtain the new timestep $t_\theta$. Subsequently, noise level at $t_\theta$ is added to $\boldsymbol{z}_0$ using \scheduler to obtain $\boldsymbol{z}_t'$. The rest is the same as in the training objective in Eq.~\ref{eq:loss}.}
    \label{fig:pipeline}
    \vspace{-2mm}
\end{figure*}
\subsection{Conditional Sequence Modeling With Diffusion Models}
Diffusion models rely on a continuous space where Gaussian noise can be smoothly added and subtracted. 
However, texts are composed of discrete tokens with no inherent notion of “small changes” between them. 
To address this, \texttt{DiffusionLM}~\cite{DBLP:conf/nips/LiTGLH22} maps the text sequence $\boldsymbol{x} \in \{0, 1\}^{L\times V}$, where each token is represented as a one-hot vector, into a continuous latent space $\boldsymbol{z}_{0}\in\mathbb{R}^{L\times H}$, with sequence length $L$, hidden dimension $H$, and vocabulary size $V$.
Hence, $\boldsymbol{z}_{0}\sim q(\boldsymbol{z}_{0}|\boldsymbol{x})$ is the embedding codebook of $\boldsymbol{x}$.

The reverse process aims to generate $\boldsymbol{z}_0$, which is then mapped back to $\boldsymbol{x}$.
The diffusion model is trained on the latent space with the objective
\begin{align}
    &\mathcal{L}_{\text{total}} = \mathcal{L}_{\text{diffusion }}+\mathcal{L}_{\text{round}} \nonumber\\
    &= \mathbb{E}_{\boldsymbol{z}_{0},t}[||D_\theta(\boldsymbol{z}_t)-\boldsymbol{z}_{0}||^{2}]+ \mathbb{E}_{\boldsymbol{z}_{0}} \left[ -\log p_\theta(\boldsymbol{x}|\boldsymbol{z}_{0}) \right] ,\label{eq:loss}
\end{align}
where $\mathcal{L}_{\text{round}}$ is the $\boldsymbol{z}_0\to\boldsymbol{x}$ reconstruction loss.

To extend the model for conditional generation, a simple yet effective approach is to incorporate the conditioning sequence $\boldsymbol{c}$ as an additional input to the denoising network, i.e., $D_\theta(\boldsymbol{z}_t, \boldsymbol{c})$. 
The target sequence length $L$ can be inferred from the conditioning context via a learned prior $L\sim p(L|\boldsymbol{c})$.
Aside from this conditioning, the diffusion process remains unchanged as in Eq.~\ref{eq:posterior}.

\subsection{Self-conditioning Diffusion Models}
\label{subsec:self-conditioning-diffusion-models}
During training, the initial forward prediction $\hat{\boldsymbol{z}}_\theta^t$ is fed back into the denoising model as an auxiliary condition.
The $\boldsymbol{z}_0$-prediction then becomes $\boldsymbol{\bar{z}}_\theta^{t}=D_\theta(\boldsymbol{{z}}_t,\mathrm{sg}(\boldsymbol{\hat{z}}_\theta^t))$, where $\mathrm{sg}(\cdot)$ denotes the stop-gradient operation, preventing gradient propagation through $\boldsymbol{\hat{z}}_\theta^t$. 
For clarity, we omit the source condition $\boldsymbol{c}$, as the discussion here focuses solely on the self-conditioning mechanism.
The self-conditioning training objective becomes:
\begin{align}
    \mathcal{L}_{\text{sc}} = \mathbb{E}_{\boldsymbol{z}_0,t}[||D_\theta(\boldsymbol{z}_t,\mathrm{sg}(\boldsymbol{\hat{z}}_\theta^t))-\boldsymbol{z}_0||^2]. 
    \label{eq:sc_obj}
\end{align}
The training process alternates between optimizing $\mathcal{L}_{\text{diffusion}}$ and $\mathcal{L}_{\text{sc}}$.

At each step, estimating $\boldsymbol{\hat{z}}_\theta^t$ followed by $\boldsymbol{\bar{z}}_\theta^{t}$ double the inference time.
Instead, the prediction from the previous step $u$ is reused as the conditioning to avoid additional inference overhead.
We denote this estimation as $\bar{\boldsymbol{z}}_\theta^{tu}$, for any $0 < s < t < u \leq 1$.

\section{Methodology}
In this section, we describe the design of \framework for efficient sequence-to-sequence language generation.
Fig.~\ref{fig:pipeline} provides an overview. Building upon the standard text diffusion architecture, \framework augments training with two tightly coupled components: {Self-conditioning Perturbation} (\scheduler) and {Model-aware Noise Scaling} (\noisescaling). 
Together, they target the two bottlenecks identified in Sec.~\ref{sec:introduction}: (i) mismatch between training-time and inference-time self-conditioning under few-step discretization, and (ii) late-stage training saturation caused by uninformative, token-unaware noise.

\subsection{Self-conditioning Limitations}
\label{subsec:selfcond_limit}
As described in Sec.~\ref{subsec:self-conditioning-diffusion-models}, using the reused estimation $\bar{z}_\theta^{tu}$ in place of the step-matched prediction 
$\bar{z}_\theta^{t}$ introduces a training-inference mismatch.
Intuitively, estimation from a coarser step $u$ carries larger uncertainty than that from step $t$~\cite{DBLP:conf/icml/BaoLSZZ22,DBLP:conf/iclr/NingLSSE24}, so $\bar{\boldsymbol{z}}_\theta^{tu}$ is more likely to deviate from the target embedding $\boldsymbol{z}_0$.
To measure the effect of this estimation gap, we report the \textit{BLEU} score on the generation of \textbf{IWSLT14} De-En dataset, using the correct $\bar{\boldsymbol{z}}_\theta$ and the original $\bar{\boldsymbol{z}}_\theta^{tu}$ self-condition sampling.
In Tab.~\ref{tab:sc_bleu}, we conduct experiments with different numbers of denoising steps (NFEs).

\begin{table}[h!]
\centering
\small
\setlength{\tabcolsep}{4.3pt} 
\begin{tabular}{lccccc}
\toprule
\textbf{Model} & \multicolumn{5}{c}{\textbf{Number of denoising steps (NFEs)}} \\
\cmidrule(lr){2-6}
& 5 & 20 & 50 & 100 & 1{,}000 \\
\midrule
\textbf{Original} & 27.85 & 29.83 & 29.97 & 30.10 & 30.12 \\
\textbf{Correct}      & 29.70 & 30.21 & 30.34 & 30.20 & 30.23 \\
\bottomrule
\end{tabular}
\caption{\textit{BLEU} scores.}
\label{tab:sc_bleu}
\end{table}

The table shows that the original self-conditioning signifies the approximation error when NFE is small, revealing a \textit{training and inference empirical gap}.
In contrast, the performance of the corrected self-condition slightly drops when NFE is reduced from $20$ to $5$, while it remains consistent for the rest NFEs.
This observation highlights the importance of a training design that aligns with the inference process, which not only boosts performance but also enables more efficient inference.
We conduct a more thorough theoretical analysis of the estimation gap in Appx.~\ref{subsec:theorem_proof}.

\subsection{Self-conditioning Perturbation}
\label{subsec:lins}

We propose \scheduler to reduce the training-inference mismatch of self-conditioning. Intuitively, \scheduler simulates inference self-condition behaviors by regularizing the correct self-condition in training, thereby improving stability and reducing drift without changing the sampling procedure.

The perturbed self-conditioning is obtained from the denoising output of a modified forward process, in which a perturbed version $\boldsymbol{z}'_t$ of $\boldsymbol{z}_t$ is introduced:
\begin{equation}
    \boldsymbol{z}'_{t} = \alpha_t \lambda_t \boldsymbol{z}_{0} + \sigma_{t}\sqrt{1+\gamma_{t}^{2}}\, \boldsymbol{\epsilon}_{t}, 
    \label{eq:perturb_forward}
\end{equation}
where $\boldsymbol{\epsilon}_t \sim \mathcal{N}(0,\mathbf{I})$.
The factors $\lambda_t$ and $\gamma_t$ corrupt the signal term and inflate the effective noise level, yielding a more aggressively noised input and thus a less reliable self-conditioning estimate. We parameterize them as linear schedules, $\lambda_t=(\lambda_{\min}-\lambda_{\max})t+\lambda_{\max}$ and $\gamma_t=(\gamma_{\min}-\gamma_{\max})t+\gamma_{\max}$, with monotonic variation over $t$.
The hyperparameter choice is guided by the forward ratios $\alpha_u/\alpha_t$ and $\sigma_u/\sigma_t$, so that $\boldsymbol{z}'_t$ matches the scale of a later state along the trajectory and approximates the conditioning statistics induced by prior estimation.

Empirically, we observe that the reused estimate behaves as an approximately Gaussian perturbation around the current prediction $\hat{\boldsymbol{z}}_\theta^{t}$, supporting the view that \scheduler captures the noise pattern introduced by inference-time self-conditioning. A simplified derivation and detailed distributional analysis are provided in Appx.~\ref{sec:sc_gauss}.

We summarize the training procedure in Appendix (see Alg.~\ref{alg:train}). Conventionally, we randomly apply self-conditioning with probability $50\%$, alternating between conditioned and unconditioned updates to keep the initial prediction meaningful.

\subsection{Model-aware Noise Scaling}
\begin{figure}[t!]
    \centering
    \begin{overpic}[width=0.98\linewidth]{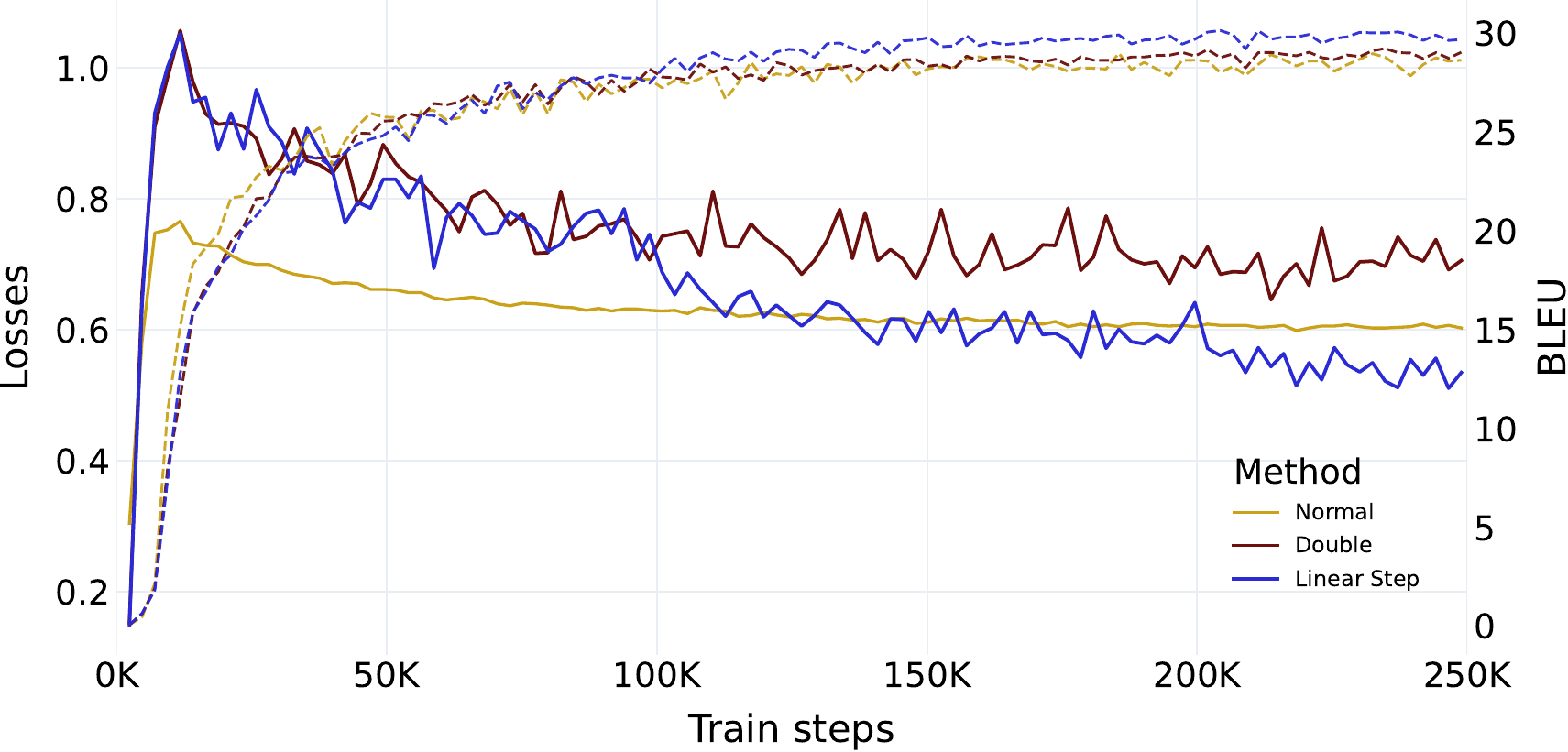}
    \end{overpic}
    \caption{Validation loss and \textit{BLEU} during training under fixed, double, and linear step noise scaling.
    Dashed lines denote \textit{BLEU} scores, color-matched to the corresponding loss curves.
    }
    \label{fig:training_t_scale}
\end{figure}

\begin{table*}[t!]
    \centering
    \small
    \setlength{\tabcolsep}{2.7pt} 
    \begin{tabular}{l | c c | c c c c c c}
        \toprule
        \multirow{2}{*}{\textbf{Method}} & \multirow{2}{*}{\textbf{MBR}} & \multirow{2}{*}{\textbf{NFE}} & \multicolumn{2}{c}{\textbf{IWSLT14}} & \multicolumn{2}{c}{\textbf{WMT14}} & \multicolumn{2}{c}{\textbf{WMT16}} \\
        & & & \textbf{DE}$\to$\textbf{EN} & \textbf{EN}$\to$\textbf{DE} & \textbf{DE}$\to$\textbf{EN} & \textbf{EN}$\to$\textbf{DE} & \textbf{RO}$\to$\textbf{EN} & \textbf{EN}$\to$\textbf{RO} \\
        \midrule
        \texttt{Transformer} & \multicolumn{1}{c}{5} & - & 33.61 & 28.30 & 30.55 & 26.85 & 33.08 & 32.86 \\
        \midrule
        \texttt{CMLM} & \multicolumn{1}{c}{5} & - & 29.41 & 24.34 & 28.71 & 23.22 &	31.13 & 31.26 \\
        \midrule
        \texttt{DiffusionLM}$^{\ddagger}$ & 50 & 20 & 29.11 & 22.91 & 19.69 & 17.41 & 30.17 & 29.39 \\
        \texttt{Diffomer}$^{\ddagger}$ & 20 & 20 & 28.01 & 23.31 & 25.30 &	23.80 &	29.37 & 29.20 \\
        \texttt{DINOISER} & 50 & 20 & {31.61} & 25.70 & {29.05} & {24.26} & 31.22 & 31.08 \\
        \texttt{DiffuSeq} & 10 & 2000 & 29.43 & - & 22.72 & - & - & -  \\
        \texttt{SeqDiffuSeq} & 10 & 2000 & 30.45 & 22.12 & 23.93 & 19.76 & - & - \\
        \texttt{AR-Diffusion} & 10 & 20 & 31.80 & {26.01} & - & - & - & -  \\
        \midrule
        \framework & 10 & 5 & 32.46 & 25.73 & 29.54 & 24.33 & 31.55 &	30.88 \\
        \framework & 20 & 5 & 32.70 & 26.02 & \underline{29.75} & \textbf{24.69} & 31.81 & 30.90 \\
        \framework & 10 & 20 & \underline{32.81} & \underline{26.29} & 29.47 & 24.50 & \underline{31.89} & \underline{31.37}        \\
        \framework & 20 & 20 & \textbf{32.88} & \textbf{26.39} & \textbf{29.83} & \underline{24.57} & \textbf{31.99} & \textbf{31.44} \\
        \bottomrule
    \end{tabular}
    \caption{Main results on \textbf{IWSLT14}, \textbf{WMT14}, and \textbf{WMT16}. The best \texttt{NAR} results are \textbf{bold} and the second-best results are \underline{underlined}. $\ddagger$ indicates results reported by \citet{DBLP:journals/corr/abs-2302-10025}. Other results are from their original papers.}
    \label{tab:iwslt14_bleu}
\end{table*}

We introduce \noisescaling to make the diffusion objective token-aware, mitigating late-stage loss saturation.
The key idea is to adaptively \emph{increase} noise for tokens the model already denoises confidently, while leaving uncertain tokens unchanged.
This strengthens supervision at higher noise levels where few-step discretization is most brittle, and improves self-conditioning even under heavier corruption.

To quantify denoising confidence, we evaluate reconstructed embedding quality $\hat{\boldsymbol{z}}_\theta^t$ by mapping it to the nearest token embeddings $\boldsymbol{e}_m$:
\begin{align}
    i &= \underset{m=1:V}{\arg\min} \ ||\boldsymbol{z}_0 - \boldsymbol{e}_m||_{2}, \\
    j &= \underset{m=1:V}{\arg\min} \ ||\hat{\boldsymbol{z}}_\theta^t - \boldsymbol{e}_m||_{2}.
\end{align}
where $i$ and $j$ denote the ground-truth and reconstructed tokens, respectively. 
We treat $i=j$ as a high-confidence token, indicating that the denoiser can already recover it reliably at noise level $t$.

We then define a \emph{noise scaling schedule} that increases denoising difficulty for these high-confidence tokens. Specifically, at training iteration $n$, we rescale the effective timestep used to corrupt the token by a model-aware factor $\beta(n)$:
\begin{equation}
    t_\theta = \begin{cases}
         \beta(n)\cdot t & \text{if} \quad i = j, \\
         t & \text{otherwise.}
    \end{cases}
\end{equation}
Intuitively, correctly reconstructed tokens are pushed to a higher-noise regime so the model continues to receive informative gradients, while incorrectly reconstructed tokens remain at the original difficulty to avoid destabilizing learning.

While $\beta(\cdot)$ can in principle take many forms, our ablations indicate that \emph{progressively increasing} the noise is crucial to overcoming loss saturation. 
As shown in Fig.~\ref{fig:training_t_scale}, a constant doubling of noise ($\beta=2.0$) yields limited gains and does not consistently break saturation, whereas a \texttt{Linear} schedule continues to improve both validation loss and BLEU even in late training.

Accordingly, we instantiate $\beta(\cdot)$ as a linear stepping schedule, where the scaling factor increases after predefined iteration milestones.
This ensures that once the model has sufficiently optimized the current noise regime, it is presented with a systematically harder denoising target, keeping the training signal non-trivial throughout optimization. Qualitative examples illustrating high and low confidence tokens are provided in Appx.~\ref{sec:visualize}.

\section{Experiments}
\label{sec:exp}
\subsection{Experimental Settings}
\paragraph{Datasets.} 
Following prior works ~\cite{DBLP:conf/iclr/GongLF0K23}, we evaluate Machine Translation on \textbf{IWSLT14} (En–De/De–En)~\cite{DBLP:conf/iwslt/CettoloNSBF14}, \textbf{WMT14} (En-De/De-En), and \textbf{WMT16} (En-Ro/Ro-En)~\cite{DBLP:conf/wmt/BojarBFHKLMPPSS14};  
Summarization on \textbf{Gigaword}~\cite{DBLP:conf/emnlp/NarayanCL18}; Question Paraphrase on \textbf{QQP}; and Text Simplification on \textbf{Wiki-Auto}.

\paragraph{Evaluation Metrics.}
We report \textit{SacreBLEU} for Machine Translation~\cite{DBLP:journals/corr/abs-2302-10025,DBLP:conf/iclr/GongLF0K23}; \textit{ROUGE-1/2/L} for Summarization~\cite{DBLP:conf/emnlp/QiYGLDCZ020}; and for Question Paraphrasing and Text Simplification, we follow the evaluation setup of~\citet{DBLP:conf/iclr/GongLF0K23}, using sentence-level \textit{BLEU}, \textit{ROUGE-L}, and \textit{BERTScore}~\cite{DBLP:conf/iclr/ZhangKWWA20} to assess quality, along with sentence-level self-\textit{BLEU}~\cite{DBLP:conf/sigir/ZhuLZGZWY18} to measure diversity.

\begin{table*}[t!]
    \setlength{\tabcolsep}{3pt} 
    \centering
    \small
    \begin{tabular}{l | c  c | c c c | c c c}
        \toprule
        \multirow{2}{*}{\textbf{Method}} & \multirow{2}{*}{\textbf{MBR}} & \multirow{2}{*}{\textbf{NFE}} & \multicolumn{3}{c|}{\textbf{QQP}} & \multicolumn{3}{c}{\textbf{Wiki-Auto}} \\
        & & & \textbf{\textit{BLEU}}$\uparrow$ & \textbf{\textit{Rouge-L}}$\uparrow$ & \textbf{\textit{BertScore}}$\uparrow$ & \textbf{\textit{BLEU}}$\uparrow$ & \textbf{\textit{Rouge-L}}$\uparrow$ & \textbf{\textit{BertScore}}$\uparrow$  \\
        \midrule
        \midrule
        \texttt{Transformer}$^{\ddagger}$ & - & - & 27.22 & 57.48 & 83.81 & 26.93 & 49.07 & 73.81 \\
        \texttt{GPT2-large FT}$^{\ddagger}$ & - & - & 20.59 & 54.15 & 83.63 & 26.93 & 51.11 & 78.82	\\
        \midrule
        \texttt{CMLM}$^{\times}$ & - & 10 & 21.78 & 56.12 & - & 35.26 & 58.46 & 81.83\\
        \texttt{LevT}$^{\ddagger}$ & - & - & 22.68 & 57.95 & 83.44 & 20.52 & 44.02 & 72.54 \\
        \midrule
        \texttt{Difformer}$^*$ & 10 & 20 & 27.95 & \underline{59.24} & 82.97 & 34.78 & 54.55 & 78.86 \\
        \texttt{DINOISER}$^{\dagger}$ & 10 & 20 & 19.49 & 53.16 & 80.36 & 23.88 & 48.21 & 67.87 \\
        \texttt{DiffuSeq} & 10 & 2000 & 24.13 & 58.80 & \underline{83.65} & 36.22 & 58.49 & 81.26 \\
        \texttt{SeqDiffuSeq} & 10 & 2000 & 24.34 &	- & \textbf{84.00} & 37.12 & - & 82.14 \\
        $\text{\texttt{DiffuSeq-v2}}^*$ & 10 & 10 & 23.07 & 58.35 & 82.36 & 26.60 & 51.33 & 77.04 \\
        $\text{\texttt{FlowSeq}}^*$ & 10 & 1 & 14.30 &	46.10 & 66.90 & 29.02 & 53.74 & 72.46 \\
        \texttt{DLM-One} & 10 & 1 & 22.13 &	57.41 & 82.97 & 36.30 & 58.39 & 80.84 \\
        \midrule
        \framework & 10 & 1 & 27.16 &	58.10 & 81.16 & 38.20 & 57.66 & 80.35 \\ 
        \framework & 10 & 2 & 27.94 & 58.47 & 81.81 & 40.23 & 59.10 & 81.60 \\ 
        \framework & 10 & 5 & \textbf{28.88} & \textbf{59.34} & 82.58 & \textbf{40.90} & \textbf{59.64} & \underline{82.16} \\
        \framework & 10 & 20 & \underline{28.32} & 58.88 & 82.62 & \underline{40.81} & \textbf{59.64} & \textbf{82.17} \\
        \bottomrule
    \end{tabular}
    \caption{Main results on \textbf{QQP} and \textbf{Wiki-Auto}. 
    The best \texttt{NAR} results are \textbf{bold} and the second-best results are \underline{underlined}. $\ddagger$, $\times$, and $\dagger$ indicate results reported by \citet{DBLP:conf/iclr/GongLF0K23}, \citet{DBLP:conf/acl/TangWZLCZ23}, and \citet{DBLP:conf/nips/ChuangHLGLCL24}, respectively. $*$ indicates reproduced results.
    Other results are from their original papers. 
    }
    \label{tab:qqp_and_wa_results}
\end{table*}

\paragraph{Baselines.} 
We consider three baseline groups: (1) Autoregressive models: \texttt{Transformer}~\cite{DBLP:conf/nips/VaswaniSPUJGKP17}, \texttt{GRU} with attention, and fine-tuned \texttt{GPT2-Large}; (2) Non-autoregressive model: \texttt{CMLM}~\cite{DBLP:conf/emnlp/GhazvininejadLL19} and \texttt{LevT}~\cite{DBLP:conf/nips/GuWZ19}; and (3) Diffusion language models: \texttt{DiffusionLM}~\cite{DBLP:conf/nips/LiTGLH22}, \texttt{Difformer}~\cite{DBLP:conf/naacl/GaoG0ZZ0X24}, \texttt{DINOISER}~\cite{DBLP:journals/corr/abs-2302-10025}, \texttt{DiffuSeq}~\cite{DBLP:conf/iclr/GongLF0K23}, \texttt{SeqDiffuSeq}~\cite{DBLP:conf/naacl/YuanYTHH24}, and \texttt{AR-Diffusion}~\cite{DBLP:conf/nips/WuFLZGS0LWGDC23}.
For Summarization, we include \texttt{LSTM}~\cite{DBLP:journals/tnn/GreffSKSS17} and \texttt{NAG-BERT}~\cite{DBLP:conf/eacl/SuCWVBLC21}.
For Question Paraphrase, we include Discrete Diffusion with \texttt{RDM}~\cite{DBLP:journals/corr/abs-2302-05737}.
We also include few-step generation benchmarks, \texttt{DiffuSeq-v2}~\cite{DBLP:conf/emnlp/GongLF0K23}, \texttt{FlowSeq}~\cite{DBLP:conf/eacl/HuWAMFOS24} and \texttt{DLM-One}~\cite{DBLP:journals/corr/abs-2506-00290}.

\paragraph{Training and Inference.}
During training, we adopt \texttt{sqrt} noise schedule~\cite{DBLP:conf/nips/LiTGLH22} with $T = 2000$ diffusion steps. 
The anchor points $(\lambda_{\min}, \lambda_{\max})$ and $(\gamma_{\min}, \gamma_{\min})$ are $(0.9,0.95)$ and $(0.15,0.35)$, respectively.
\noisescaling is randomly applied with $50\%$ probability to fasten training. 
Our implementation is based on \texttt{Difformer}, with the same sampling configurations at $\text{NFE}\in\{2, 5, 20\}$. 
For every task, we construct the vocabulary with Byte Pair Encoding~\cite{DBLP:conf/emnlp/KudoR18}.
We also apply Minimum Bayes Risk (MBR) decoding~\cite{DBLP:conf/naacl/KumarB04} following previous works~\cite{DBLP:conf/nips/LiTGLH22,DBLP:conf/iclr/GongLF0K23}.
Further details are described in Appx.~\ref{sec:exp_conf} and~\ref{sec:mbr_abl}.

\subsection{Main Results}

\paragraph{Overall Performance.}
Tabs.~\ref{tab:iwslt14_bleu} and~\ref{tab:qqp_and_wa_results} summarize results across tasks and datasets.
Overall, \framework improves few-step diffusion generation, outperforming both \texttt{Diffusion} and \texttt{Non-autoregressive} baselines on most settings (varying MBR and NFEs), while approaching \texttt{Autoregressive} performance.
On \textbf{WMT14}, the $5$-step model even surpasses $20$-step sampling, indicating that \scheduler can effectively close the gap between few and many inference.

In terms of efficiency, \framework is $4\times$ faster than \texttt{DINOISER}, \texttt{Difformer}, and up to $400\times$ faster than long-trajectory methods such as \texttt{DiffuSeq} and \texttt{SeqDiffuSeq}.
Compared to one-step baselines, \framework remains competitive without relying on distillation (\texttt{DLM-One}) or flow-matching (\texttt{FlowSeq}) objectives (Tab.~\ref{tab:qqp_and_wa_results}).
Additional results on Text Summarization are reported in Appx.~\ref{sec:add_result}.

\paragraph{Sampling Speed.}

We analyze the quality--speed trade-off on \textbf{QQP} by comparing \framework with \texttt{DiffuSeq-v2} using its original self-conditioning mechanism~\cite{DBLP:conf/emnlp/GongLF0K23}.
Fig.~\ref{fig:speed-bleu} shows that \framework consistently achieves higher quality at low NFEs, highlighting its advantage in the few-step regime.
The margin narrows at medium NFEs, where \texttt{DiffuSeq-v2} slightly overtakes \framework, suggesting that the benefit of \scheduler diminishes once discretization error becomes small.

\paragraph{Sampling Diversity.}
We evaluate diversity on \textbf{QQP} using \textit{BLEU} and self-\textit{BLEU}.
Fig.~\ref{fig:self-bleu} shows the trade-off between diversity and quality.
Fig.~\ref{fig:self-bleu} shows that \framework-$2$NFE matches the strongest baseline, \texttt{Difformer}-$20$NFE, with 10$\times$ smaller bugdet.
Increasing NFEs improves quality at the expected cost of latency.
Notably, \framework-$5$NFE achieves substantially higher \textit{BLEU} while maintaining similar self-\textit{BLEU}, indicating that quality gains do not come at the expense of diversity.

\begin{figure}[t]
    \centering
    \includegraphics[width=0.7\linewidth]{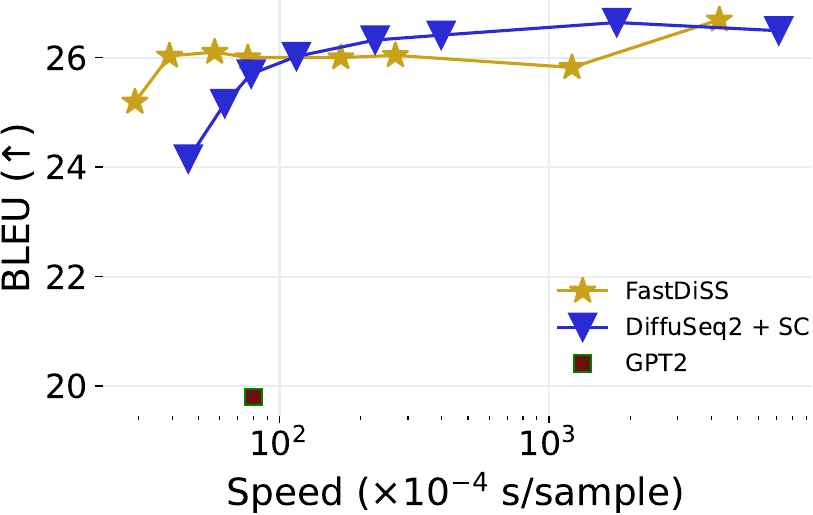}
    \caption{Generation speed and quality with different NFE. The speed is averaged over 3 runs.}
    \label{fig:speed-bleu}
\end{figure}

\begin{figure}[t]
    \centering
    \includegraphics[width=0.7\linewidth]{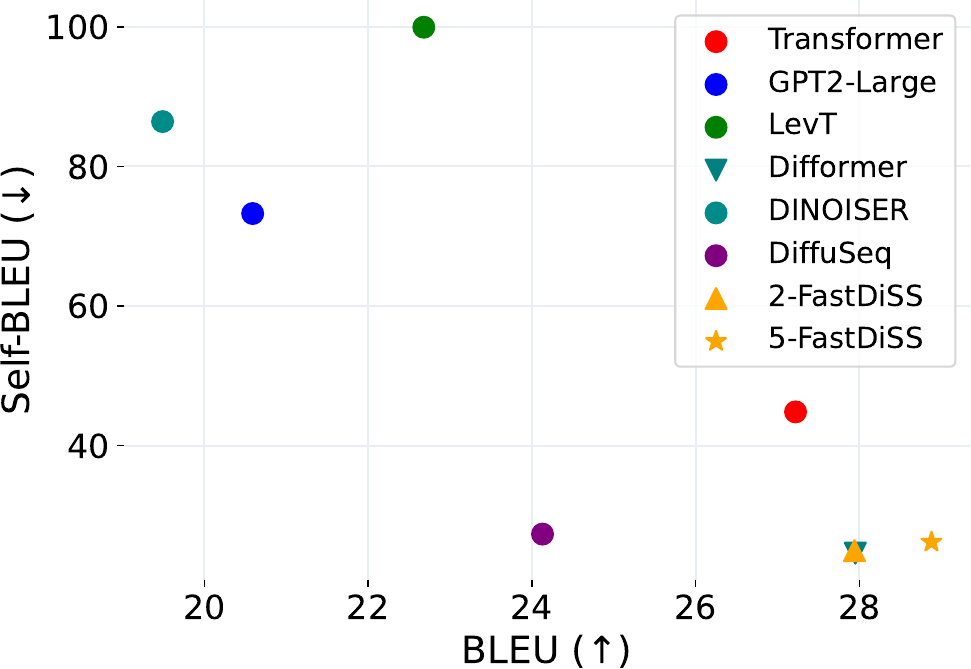}
    \caption{Diversity and quality comparison.}
    \label{fig:self-bleu}
\end{figure}

\subsection{Ablation Studies}
\textbf{Effect of Each Component.}
We quantify the contribution of \scheduler and \noisescaling in Tab.~\ref{tab:abl_comp}. 
The evaluation is conducted on \textbf{IWSLT14} De-En using \textit{BLEU}.
The results show that both components provide consistent gains over the base model.
\noisescaling is particularly effective at small NFEs, suggesting it improves prediction at large steps, reducing self-conditioning errors.
In contrast, \scheduler mainly improves inference by reducing the training-inference gap, which is most visible in the few-step regime.
Combining \scheduler and \noisescaling yields the strongest performance and significantly narrows the gap between 5-step and 20-step sampling.

\begin{table}[t]
    \centering
    \small
    \setlength{\tabcolsep}{4.5pt} 
    \renewcommand{\arraystretch}{0.8} 
    \begin{tabular}{c c | c c c}
        \toprule
        \textbf{\scheduler} & \textbf{\noisescaling} & $\text{NFE}=5$ & $\text{NFE}=20$ & $\Delta$NFE \\
        \midrule
        $\times$ & $\times$ & 27.98 & 29.78 & 1.80 \\
        $\checkmark$ & $\times$ & 29.64 & 30.36 & 0.72 \\
        $\times$ & $\checkmark$ & 30.77 & 31.49 & 0.72 \\
        $\checkmark$ & $\checkmark$ & 31.17 & 31.66 & 0.49 \\
        \bottomrule
    \end{tabular}
    \caption{Ablation study.}
    \label{tab:abl_comp}
\end{table}

\paragraph{Effect of Varying NFEs.}
We vary NFEs on \textbf{IWSLT14} De-En and compare \framework against the original codebase, \texttt{Difformer}.
Fig.~\ref{fig:abl_nfe} shows that \framework is markedly more effective under few-step inference: it surpasses the $20$-step baseline within $3$ sampling steps, corresponding to approximately $7\times$ speedup.
Moreover, performance converges after roughly $7$ steps, reaching a higher \textit{BLEU} than the $20$-step baseline.

\begin{figure}[t]
    \centering
    \includegraphics[width=\linewidth]{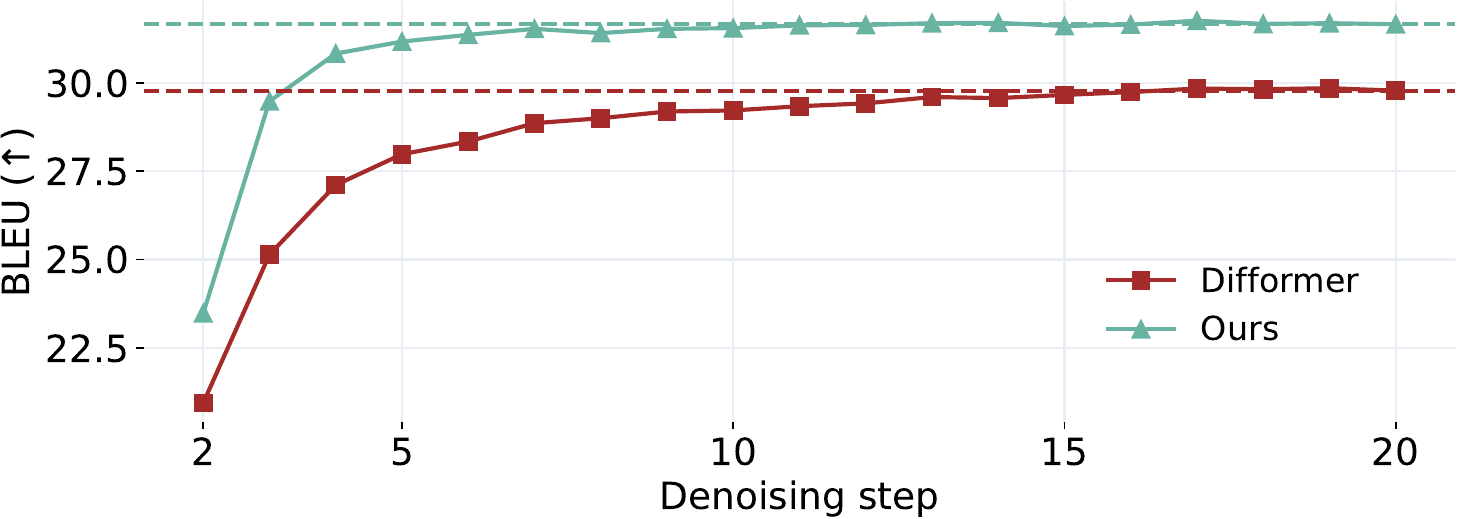}
    \caption{Effect of Number of Denoising Steps.}
    \label{fig:abl_nfe}
\end{figure}

\paragraph{Effect of $\gamma_t$ and $\lambda_t$.}
We study sensitivity to $\gamma_t$ and $\lambda_t$ in \scheduler (Eq.~\ref{eq:perturb_forward}) on \textbf{QQP}.
Tab.~\ref{tab:abl_sched_var} compares linear-time schedules against fixed variants where $\lambda_t$ and $\gamma_t$ are held constant.
The best results are obtained with settings derived from the 20-step ratios, where $\gamma_t$ ranges from $0.90$ to $0.95$ and $\lambda_t$ ranges from $0.15$ to $0.35$.
We use this configuration in the remaining experiments.

\begin{table}[t!]
    \centering
    \small
    \setlength{\tabcolsep}{2.2pt} 
    \begin{tabular}{l c c c c}
        \toprule
        \textbf{Steps} & 5 & 20 & 100 & Fixed \\
        \midrule
        \textbf{$\lambda_{t}$} & 0.60 - 0.85 & 0.90 - 0.95 & 0.98 - 0.99 & 0.85 \\
        \textbf{$\gamma_{t}$} & 0.25 - 0.90 & 0.15 - 0.35 & 0.05 - 0.15 & 0.50 \\
        \midrule
         \textbf{\textit{BLEU}} & 25.48 & \textbf{26.35} &  25.70 & 25.84 \\
         \textbf{\textit{ROUGE-L}} & 57.29 & \textbf{57.47} & 56.82 & 57.18 \\
        \bottomrule
    \end{tabular}
    \caption{Effect of $\gamma_t$ and $\lambda_t$.}
    \label{tab:abl_sched_var}
\end{table}

\paragraph{Effect of Noise Schedulers.}
Finally, we compare \framework under standard diffusion noise schedules, including \texttt{linear}~\cite{DBLP:conf/nips/HoJA20} and \texttt{cosine}~\cite{DBLP:conf/icml/NicholD21}, on \textbf{QQP}.
Tab.~\ref{tab:abl_sched} shows that the \texttt{linear} schedule is generally less sensitive to the choice of NFE.
\framework consistently yields improvements in the few-step regime across schedulers, confirming that our gains are not tied to a specific base schedule.

\begin{table}[t]
    \centering
    \small
    \setlength{\tabcolsep}{5.5pt} 
    \begin{tabular}{l c c c c}
        \toprule
        \multicolumn{2}{c}{\textbf{Noise Schedule}} & NFE=2 & NFE=5 & NFE=20 \\
        \midrule
         \multirow{2}{*}{\textbf{\texttt{linear}}} & \textbf{Orig.} & 26.50 & 27.17 & \textbf{27.54} \\
         & \textbf{Ours} & \textbf{26.84} & \textbf{27.52} & \textbf{27.54} \\
         \midrule
         \multirow{2}{*}{\textbf{\texttt{cosine}}} & \textbf{Orig.} &  25.43 & 27.05 & 27.57 \\
         & \textbf{Ours} & \textbf{27.32} & \textbf{27.77} & \textbf{27.94} \\
        \bottomrule
    \end{tabular}
    \caption{Effect of Noise Schedulers.}
    \label{tab:abl_sched}
\end{table}

\subsection{Extension to Large-scale Benchmark}
\label{subsec:gsm8k}
We evaluate \scheduler adaptability to reasoning benchmark \textbf{GSM8K}~\cite{DBLP:journals/corr/abs-2110-14168} and discrete setting in Tabs.~\ref{tab:gsm8k} and~\ref{tab:mdlm_result}.

\paragraph{Reasoning benchmark.} 
We adopt \texttt{DoT}~\cite{DBLP:conf/nips/YeGCZGSWJLBK24}, using \texttt{Plaid}-1B~\cite{DBLP:conf/nips/GulrajaniH23} as the backbone. 
We introduce \scheduler as a lightweight modification following \texttt{Plaid} training recipe. 
Tab.~\ref{tab:gsm8k} shows that \scheduler consistently improves accuracy under both standard diffusion inference and Chain-of-Thought (\texttt{CoT}) inference.
More importantly, when reducing NFEs, \scheduler substantially mitigates the drop in accuracy, indicating improved robustness in the low-compute regime.
Qualitative examples comparing \texttt{Plaid} with and without \texttt{SCP} are provided in Appx.~\ref{sec:qualitative_res}.


\begin{table}[t]
\centering
\small
\setlength{\tabcolsep}{3.6pt} 
\begin{tabular}{rccp{0.005mm}cc}
\toprule
\multirow{2}{*}{\textbf{NFE}} & \multicolumn{2}{c}{\textbf{\texttt{DoT} (\texttt{Plaid})}}& & \multicolumn{2}{c}{\textbf{$\texttt{DoT}^{\text{mp}}$ (\texttt{CoT})}} \\
\cmidrule{2-3}\cmidrule{5-6}
 & \textbf{Normal} & \textbf{\scheduler} & & \textbf{Normal} & \textbf{\scheduler} \\
\midrule
8  & 35.4 $\pm$ 0.5 & $\mathbf{38.2 \pm 0.7}$  & & 30.9 $\pm$ 0.6  & $\mathbf{35.1 \pm 1.0}$ \\
16 & 39.0 $\pm$ 0.5 & $\mathbf{41.1 \pm 0.7}$ & & 36.1 $\pm$ 0.7  & $\mathbf{39.8 \pm 0.8}$ \\
32 & 39.9 $\pm$ 0.3 & $\mathbf{43.0 \pm 1.1}$ & & 37.2 $\pm$ 0.7  & $\mathbf{40.4 \pm 0.9}$ \\
64 & 41.3 $\pm$ 0.2 & $\mathbf{43.7 \pm 0.2}$ & & 36.6 $\pm$ 0.6 & $\mathbf{40.8 \pm 0.6}$ \\
\bottomrule
\end{tabular}
\caption{Comparison of \textit{Accuracy} on \textbf{GSM8K} benchmark. We experimented with two settings, \texttt{DoT} with \texttt{Plaid} inference and $\texttt{DoT}^{\text{mp}}$ with \texttt{CoT} reasoning. The \textit{Accuracy} is averaged over 5 runs.}
\label{tab:gsm8k}
\end{table}

\paragraph{Discrete Diffusion Model Extension.} 
\texttt{MDLM}, the masking diffusion framework behind \texttt{LLaDA}~\cite{DBLP:journals/corr/abs-2502-09992}, operates via iterative masking and unmasking.
In this setting, \scheduler corresponds to increasing the masking rate during training, producing a more corrupted conditioning context that better matches few-step inference. 
As shown in Tab.~\ref{tab:mdlm_result}, \scheduler improves robustness at small NFEs compared to self-conditioning, and yields consistent gains over the original model as NFEs increase.

\begin{table}[t]
\centering
\small
\begin{tabular}{rrrr}
\toprule
\multirow{2}{*}{\textbf{NFE}} & \multicolumn{3}{c}{\textbf{Gen Perplexity ($\downarrow$)}} \\
\cmidrule(lr){2-4}
 & \textbf{None} & \textbf{SC} & \textbf{\scheduler} \\
\midrule
2    & 3178.74 & 3250.22 & \textbf{2873.67} \\
5    & 1144.42 & 1433.79 & \textbf{1215.41} \\
20   & 243.81  & 254.72  & \textbf{244.40} \\
50   & 143.82  & 133.79  & \textbf{130.52} \\
100  & 113.34  & 97.88   & \textbf{95.48}  \\
1000 & 55.23   & 45.95   & \textbf{45.21}  \\
5000 & 32.12   & \textbf{27.78}   & 28.14 \\
\bottomrule
\end{tabular}
\caption{Main results on \texttt{MDLM}. 
    The best \texttt{Generation Perplexity} results are \textbf{bold}.}
\label{tab:mdlm_result}
\end{table}
\section{Related Works}
\label{sec:relatedwork}
\textbf{Non-autoregressive Language Generation.}
Non-autoregressive models were early introduced to reduce generation latency by predicting tokens in parallel~\cite{DBLP:journals/corr/abs-1711-02281}. 
To improve accuracy, later work incorporated iterative refinement and editing-style decoding~\cite{DBLP:conf/nips/GuWZ19,DBLP:conf/emnlp/GhazvininejadLL19}. 
Despite these advances, the conditional independence assumption in many \texttt{NAR} formulations leads to a single mode selection issue. 
Prior efforts mitigate this issue via structured prediction~\cite{DBLP:conf/naacl/ZhangW0G0QL22,DBLP:conf/acl/RanLLZ20,DBLP:conf/icml/HuangTZLH22}, data augmentation and selection (e.g., reference rephrasing)~\cite{DBLP:conf/naacl/ShaoW022,DBLP:conf/aaai/ShaoZ0023}, and distillation to transfer knowledge from autoregressive teachers~\cite{DBLP:journals/corr/abs-2112-11640,DBLP:journals/corr/abs-2205-11162,DBLP:conf/aaai/LiuBZH23}.

\noindent\textbf{Diffusion Models for Language Generation.}
Text diffusion models can be broadly categorized into \emph{discrete} and \emph{continuous} formulations. Discrete diffusion defines Markov transitions directly over token space. Early approaches such as \texttt{D3PM}~\cite{DBLP:conf/nips/AustinJHTB21} and \texttt{Analog-bits}~\cite{DBLP:conf/iclr/ChenZH23} rely on absorbing or uniform transitions, while more recent methods adopt masking-style transitions that scale more naturally to language, including \texttt{MDLM}~\cite{DBLP:conf/nips/SahooASGMCRK24} and \texttt{RDM}~\cite{DBLP:journals/corr/abs-2302-05737}. 
These advances have enabled large-scale diffusion LLMs such as \texttt{LLaDA}~\cite{DBLP:journals/corr/abs-2502-09992}, \texttt{Dream-7B}~\cite{DBLP:journals/corr/abs-2508-15487}, and \texttt{Block Diffusion}~\cite{DBLP:conf/iclr/ArriolaGCYQHSK25}.
In contrast, continuous diffusion maps tokens into a continuous embedding space and applies Gaussian processes, enabling denoising in latent space.
Representative models include \texttt{DiffusionLM}~\cite{DBLP:conf/nips/LiTGLH22}, \texttt{Difformer}~\cite{DBLP:conf/naacl/GaoG0ZZ0X24}, \texttt{DINOISER}~\cite{DBLP:journals/corr/abs-2302-10025}, \texttt{SeqDiffuSeq}~\cite{DBLP:conf/naacl/YuanYTHH24}, \texttt{AR-Diffusion}~\cite{DBLP:conf/nips/WuFLZGS0LWGDC23}, and \texttt{DiffuSeq}~\cite{DBLP:conf/iclr/GongLF0K23}, spanning encoder-decoder and decoder-only designs. 
Several works further tailor diffusion to language by modifying the noise schedule or corruption process, including \texttt{Masked-Diffuse}~\cite{DBLP:conf/emnlp/ChenZ0SY23} and \texttt{Meta-Diffu$B$}~\cite{DBLP:conf/nips/ChuangHLGLCL24}.
Our method is developed in the continuous latent setting to explore the full potential of continuous modeling.

\noindent\textbf{Accelerating the Diffusion Language Model.}
Widely adopted methods, such as \texttt{DDIM}~\cite{DBLP:conf/iclr/SongME21}, have demonstrated success in prior text generation works. 
Advances in ODE solvers~\cite{DBLP:conf/nips/0011ZB0L022,DBLP:journals/ijautcomp/LuZBCLZ25} have made great success in enhancing the efficiency of the reverse process in \texttt{DiffuSeq-v2}~\cite{DBLP:conf/emnlp/GongLF0K23}. ~\citet{DBLP:conf/acl/TangWZLCZ23} also addresses training-inference discrepancies to enhance generation speed. 
Recently, \texttt{FlowSeq}~\cite{DBLP:conf/eacl/HuWAMFOS24} adapts flow matching for sequence modeling, while \texttt{DLM-One}~\cite{DBLP:journals/corr/abs-2506-00290} distills the teacher model, \texttt{DiffuSeq}, to learn a one-step generator.
Our work advances this line by promoting fast, adaptable, and few-step generation, aiming to make diffusion-based language models more viable for real-world applications.

\section{Conclusion}

We introduced \framework, a novel training framework for diffusion language models, combining {Self-conditioning Perturbation} (\scheduler) to align self-conditioning under few-step discretization and {Model-aware Noise Scaling} (\noisescaling) to increase noise on high-confidence tokens for more informative training.
Experiments on six benchmarks demonstrate that \framework enables more effective training, narrows the gap between few-step and many-step decoding, and achieves significant improvements in both efficiency and generation quality.
While our method advances the practicality of fast inference with diffusion models, future work will focus on further refining self-conditioning and closing the gap with autoregressive approaches.

\newpage
\section*{Limitation}
\label{sec:limit}
Since we aim to align training to inference, the proposed techniques are constrained by the goodness of the previous self-conditioning prediction. It would be more promising if we approached the problem from the opposite direction, as refining the prediction to be more accurate would significantly boost the performance. Next, our current noise scaling strategy relies on predefined values at each training phase, which may be suboptimal when earlier scaling stages are insufficiently trained. A more adaptive scaling function could further enhance performance.

\section*{Broader Impact}
This work advances diffusion-based language modeling by proposing mechanisms that improve training efficiency and training-inference alignment. By reducing exposure bias and accelerating convergence, our approach contributes to the development of faster, more robust text generation systems. These improvements can benefit a wide range of natural language applications, including machine translation, summarization, and question answering, where efficiency and quality are critical.

Despite these potential benefits, diffusion-based models, like other large generative models, may still propagate social biases or produce harmful or misleading content when trained on unfiltered data. We emphasize the importance of responsible deployment, including careful dataset selection, bias evaluation, and human oversight. 



%
\bibliography{custom}

@inproceedings{DBLP:conf/nips/HoJA20,
  author       = {Jonathan Ho and
                  Ajay Jain and
                  Pieter Abbeel},
  title        = {Denoising Diffusion Probabilistic Models},
  booktitle    = {Proceedings of the Annual Conference
                  on Neural Information Processing Systems (NeurIPS)},
  year         = {2020},
}

@article{DBLP:journals/corr/abs-2302-10025,
  author       = {Jiasheng Ye and
                  Zaixiang Zheng and
                  Yu Bao and
                  Lihua Qian and
                  Mingxuan Wang},
  title        = {{DINOISER:} Diffused Conditional Sequence Learning by Manipulating Noises},
  journal      = {CoRR},
  volume       = {abs/2302.10025},
  year         = {2023},
}

@inproceedings{DBLP:conf/nips/WuFLZGS0LWGDC23,
  author       = {Tong Wu and
                  Zhihao Fan and
                  Xiao Liu and
                  Hai{-}Tao Zheng and
                  Yeyun Gong and
                  Yelong Shen and
                  Jian Jiao and
                  Juntao Li and
                  Zhongyu Wei and
                  Jian Guo and
                  Nan Duan and
                  Weizhu Chen},
  title        = {{AR-Diffusion:} Auto-Regressive Diffusion Model for Text Generation},
  booktitle    = {Proceedings of the Annual Conference on Neural Information Processing Systems (NeurIPS)},
  year         = {2023},
}

@inproceedings{DBLP:conf/iclr/GongLF0K23,
  author       = {Shansan Gong and
                  Mukai Li and
                  Jiangtao Feng and
                  Zhiyong Wu and
                  Lingpeng Kong},
  title        = {{DiffuSeq:} Sequence to Sequence Text Generation with Diffusion Models},
  booktitle    = {Proceedings of the International Conference on Learning Representations (ICLR)},
  year         = {2023},
}

@inproceedings{DBLP:conf/nips/LiTGLH22,
  author       = {Xiang Lisa Li and
                  John Thickstun and
                  Ishaan Gulrajani and
                  Percy Liang and
                  Tatsunori B. Hashimoto},
  title        = {{Diffusion-LM} Improves Controllable Text Generation},
  booktitle    = {Proceedings of the Annual Conference on Neural Information Processing Systems (NeurIPS)},
  year         = {2022},
}

@inproceedings{DBLP:conf/icml/NicholD21,
  author       = {Alexander Quinn Nichol and
                  Prafulla Dhariwal},
  title        = {Improved Denoising Diffusion Probabilistic Models},
  booktitle    = {Proceedings of the International Conference on Machine Learning (ICML)},
  pages        = {8162--8171},
  year         = {2021},
}

@inproceedings{DBLP:conf/naacl/GaoG0ZZ0X24,
  author       = {Zhujin Gao and
                  Junliang Guo and
                  Xu Tan and
                  Yongxin Zhu and
                  Fang Zhang and
                  Jiang Bian and
                  Linli Xu},
  title        = {Empowering Diffusion Models on the Embedding Space for Text Generation},
  booktitle    = {Proceedings of the Conference of the North American Chapter of the Association for Computational Linguistics: Human Language Technologies (NAACL)},
  pages        = {4664--4683},
  year         = {2024},
}

@inproceedings{DBLP:conf/nips/ChuangHLGLCL24,
  author       = {Yun{-}Yen Chuang and
                  Hung{-}Min Hsu and
                  Kevin Lin and
                  Chen{-}Sheng Gu and
                  Ling Zhen Li and
                  Ray{-}I Chang and
                  Hung{-}yi Lee},
  title        = {{Meta-DiffuB:} {A} Contextualized Sequence-to-Sequence Text Diffusion
                  Model with Meta-Exploration},
  booktitle    = {Proceedings of the Annual Conference on Neural Information Processing Systems (NeurIPS)},
  year         = {2024},
}

@inproceedings{DBLP:conf/iclr/NingLSSE24,
  author       = {Mang Ning and
                  Mingxiao Li and
                  Jianlin Su and
                  Albert Ali Salah and
                  Itir {\"{O}}nal Ertugrul},
  title        = {Elucidating the Exposure Bias in Diffusion Models},
  booktitle    = {Proceedings of the International Conference on Learning Representations (ICLR)},
  year         = {2024},
}

@inproceedings{DBLP:conf/iclr/0011SKKEP21,
  author       = {Yang Song and
                  Jascha Sohl{-}Dickstein and
                  Diederik P. Kingma and
                  Abhishek Kumar and
                  Stefano Ermon and
                  Ben Poole},
  title        = {Score-Based Generative Modeling through Stochastic Differential Equations},
  booktitle    = {Proceedings of the International Conference on Learning Representations (ICLR)},
  year         = {2021},
}

@inproceedings{DBLP:conf/iclr/ChenZH23,
  author       = {Ting Chen and
                  Ruixiang Zhang and
                  Geoffrey E. Hinton},
  title        = {Analog Bits: Generating Discrete Data using Diffusion Models with
                  Self-Conditioning},
  booktitle    = {Proceedings of the International Conference on Learning Representations (ICLR)},
  year         = {2023},
}

@inproceedings{DBLP:conf/icml/BaoLSZZ22,
  author       = {Fan Bao and
                  Chongxuan Li and
                  Jiacheng Sun and
                  Jun Zhu and
                  Bo Zhang},
  title        = {Estimating the Optimal Covariance with Imperfect Mean in Diffusion Probabilistic Models},
  booktitle    = {Proceedings of the International Conference on Machine Learning (ICML)},
  pages        = {1555--1584},
  year         = {2022},
}

@inproceedings{DBLP:conf/nips/GulrajaniH23,
  author       = {Ishaan Gulrajani and
                  Tatsunori B. Hashimoto},
  editor       = {Alice Oh and
                  Tristan Naumann and
                  Amir Globerson and
                  Kate Saenko and
                  Moritz Hardt and
                  Sergey Levine},
  title        = {Likelihood-Based Diffusion Language Models},
  booktitle    = {Proceedings of the Annual Conference on Neural Information Processing Systems (NeurIPS)},
  year         = {2023},
}

@inproceedings{DBLP:conf/naacl/YuanYTHH24,
  author       = {Hongyi Yuan and
                  Zheng Yuan and
                  Chuanqi Tan and
                  Fei Huang and
                  Songfang Huang},
  title        = {Text Diffusion Model with Encoder-Decoder Transformers for Sequence-to-Sequence Generation},
  booktitle    = {Proceedings of the Conference of the North American Chapter of the Association for Computational Linguistics: Human Language Technologies (NAACL)},
  pages        = {22--39},
  year         = {2024},
}

@inproceedings{DBLP:conf/icml/NingSPCC23,
  author       = {Mang Ning and
                  Enver Sangineto and
                  Angelo Porrello and
                  Simone Calderara and
                  Rita Cucchiara},
  title        = {Input Perturbation Reduces Exposure Bias in Diffusion Models},
  booktitle    = {Proceedings of the International Conference on Machine Learning (ICML)},
  pages        = {26245--26265},
  year         = {2023},
}

@inproceedings{DBLP:conf/acl/TangWZLCZ23,
  author       = {Zecheng Tang and
                  Pinzheng Wang and
                  Keyan Zhou and
                  Juntao Li and
                  Ziqiang Cao and
                  Min Zhang},
  title        = {Can Diffusion Model Achieve Better Performance in Text Generation? Bridging the Gap between Training and Inference!},
  booktitle    = {Findings of the ACL},
  pages        = {11359--11386},
  year         = {2023},
}

@inproceedings{DBLP:conf/emnlp/ChenZ0SY23,
  author       = {Jiaao Chen and
                  Aston Zhang and
                  Mu Li and
                  Alex Smola and
                  Diyi Yang},
  editor       = {Houda Bouamor and
                  Juan Pino and
                  Kalika Bali},
  title        = {A Cheaper and Better Diffusion Language Model with Soft-Masked Noise},
  booktitle    = {Proceedings of the Conference on Empirical Methods in Natural Language Processing (EMNLP)},
  pages        = {4765--4775},
  year         = {2023},
}

@inproceedings{DBLP:conf/wmt/BojarBFHKLMPPSS14,
  author       = {Ondrej Bojar and
                  Christian Buck and
                  Christian Federmann and
                  Barry Haddow and
                  Philipp Koehn and
                  Johannes Leveling and
                  Christof Monz and
                  Pavel Pecina and
                  Matt Post and
                  Herve Saint{-}Amand and
                  Radu Soricut and
                  Lucia Specia and
                  Ales Tamchyna},
  title        = {Findings of the 2014 Workshop on Statistical Machine Translation},
  booktitle    = {Proceedings of the Workshop on Statistical Machine Translation (WMT@ACL)},
  pages        = {12--58},
  year         = {2014},
}

@inproceedings{DBLP:conf/iwslt/CettoloNSBF14,
  author       = {Mauro Cettolo and
                  Jan Niehues and
                  Sebastian St{\"{u}}ker and
                  Luisa Bentivogli and
                  Marcello Federico},
  title        = {Report on the 11th {IWSLT} Evaluation Campaign},
  booktitle    = {Proceedings of the International Workshop on Spoken Language Translation: Evaluation Campaign (IWSLT)},
  pages        = {2--17},
  year         = {2014},
}

@inproceedings{DBLP:conf/emnlp/NarayanCL18,
  author       = {Shashi Narayan and
                  Shay B. Cohen and
                  Mirella Lapata},
  title        = {Don't Give Me the Details, Just the Summary! Topic-Aware Convolutional
                  Neural Networks for Extreme Summarization},
  booktitle    = {Proceedings of the Conference on Empirical Methods in Natural Language Processing EMNLP},
  pages        = {1797--1807},
  year         = {2018},
}

@inproceedings{DBLP:conf/nips/VaswaniSPUJGKP17,
  author       = {Ashish Vaswani and
                  Noam Shazeer and
                  Niki Parmar and
                  Jakob Uszkoreit and
                  Llion Jones and
                  Aidan N. Gomez and
                  Lukasz Kaiser and
                  Illia Polosukhin},
  title        = {Attention is All you Need},
  booktitle    = {Proceedings of the Annual Conference on Neural Information Processing Systems (NeurIPS)},
  pages        = {5998--6008},
  year         = {2017},
}

@inproceedings{DBLP:conf/nips/GuWZ19,
  author       = {Jiatao Gu and
                  Changhan Wang and
                  Junbo Zhao},
  title        = {Levenshtein Transformer},
  booktitle    = {Proceedings of the Annual Conference on Neural Information Processing Systems (NeurIPS)},
  pages        = {11179--11189},
  year         = {2019},
}

@article{DBLP:journals/corr/abs-1711-02281,
  author       = {Jiatao Gu and
                  James Bradbury and
                  Caiming Xiong and
                  Victor O. K. Li and
                  Richard Socher},
  title        = {Non-Autoregressive Neural Machine Translation},
  journal      = {CoRR},
  volume       = {abs/1711.02281},
  year         = {2017},
}

@article{DBLP:journals/tnn/GreffSKSS17,
  author       = {Klaus Greff and
                  Rupesh Kumar Srivastava and
                  Jan Koutn{\'{\i}}k and
                  Bas R. Steunebrink and
                  J{\"{u}}rgen Schmidhuber},
  title        = {{LSTM:} {A} Search Space Odyssey},
  journal      = {{IEEE} Trans. Neural Networks Learn. Syst.},
  volume       = {28},
  number       = {10},
  pages        = {2222--2232},
  year         = {2017},
}

@article{DBLP:journals/corr/abs-2302-05737,
  author       = {Lin Zheng and
                  Jianbo Yuan and
                  Lei Yu and
                  Lingpeng Kong},
  title        = {A Reparameterized Discrete Diffusion Model for Text Generation},
  journal      = {CoRR},
  volume       = {abs/2302.05737},
  year         = {2023},
}

@inproceedings{DBLP:conf/emnlp/GongLF0K23,
  author       = {Shansan Gong and
                  Mukai Li and
                  Jiangtao Feng and
                  Zhiyong Wu and
                  Lingpeng Kong},
  title        = {{DiffuSeq-v2:} Bridging Discrete and Continuous Text Spaces for Accelerated Seq2Seq Diffusion Models},
  booktitle    = {Findings of the EMNLP},
  pages        = {9868--9875},
  year         = {2023},
}

@article{radford2018improving,
  title={Improving language understanding by generative pre-training},
  author={Radford, Alec and Narasimhan, Karthik and Salimans, Tim and Sutskever, Ilya and others},
  year={2018},
  publisher={San Francisco, CA, USA}
}

@inproceedings{DBLP:conf/iclr/ArriolaGCYQHSK25,
  author       = {Marianne Arriola and
                  Aaron Gokaslan and
                  Justin T. Chiu and
                  Zhihan Yang and
                  Zhixuan Qi and
                  Jiaqi Han and
                  Subham Sekhar Sahoo and
                  Volodymyr Kuleshov},
  title        = {Block Diffusion: Interpolating Between Autoregressive and Diffusion
                  Language Models},
  booktitle    = {Proceedings of the International Conference on Learning Representations (ICLR)},
  year         = {2025},
}

@inproceedings{DBLP:conf/emnlp/Schmidt19,
  author       = {Florian Schmidt},
  title        = {Generalization in Generation: {A} closer look at Exposure Bias},
  booktitle    = {Proceedings of the Workshop on Neural Generation and Translation (GNT@EMNLP-IJCNLP)},
  pages        = {157--167},
  year         = {2019},
}

@inproceedings{DBLP:conf/acl/PapineniRWZ02,
  author       = {Kishore Papineni and
                  Salim Roukos and
                  Todd Ward and
                  Wei{-}Jing Zhu},
  title        = {Bleu: a Method for Automatic Evaluation of Machine Translation},
  booktitle    = {Proceedings of the Annual Meeting of the Association for Computational Linguistics (ACL)},
  pages        = {311--318},
  year         = {2002},
}

@inproceedings{DBLP:conf/emnlp/QiYGLDCZ020,
  author       = {Weizhen Qi and
                  Yu Yan and
                  Yeyun Gong and
                  Dayiheng Liu and
                  Nan Duan and
                  Jiusheng Chen and
                  Ruofei Zhang and
                  Ming Zhou},
  title        = {ProphetNet: Predicting Future {N}-gram for Sequence-to-Sequence Pre-training},
  booktitle    = {Findings of EMNLP},
  pages        = {2401--2410},
  year         = {2020},
}

@inproceedings{DBLP:conf/emnlp/KudoR18,
  author       = {Taku Kudo and
                  John Richardson},
  title        = {SentencePiece: {A} simple and language independent subword tokenizer
                  and detokenizer for Neural Text Processing},
  booktitle    = {Proceedings of the Conference on Empirical Methods in Natural Language Processing (EMNLP)},
  pages        = {66--71},
  year         = {2018},
}

@inproceedings{DBLP:conf/emnlp/GhazvininejadLL19,
  author       = {Marjan Ghazvininejad and
                  Omer Levy and
                  Yinhan Liu and
                  Luke Zettlemoyer},
  title        = {Mask-Predict: Parallel Decoding of Conditional Masked Language Models},
  booktitle    = {Proceedings of the Conference on Empirical Methods in Natural Language Processing and the International Joint Conference on Natural Language Processing (EMNLP-IJCNLP)},
  pages        = {6111--6120},
  year         = {2019},
}

@inproceedings{DBLP:conf/naacl/KumarB04,
  author       = {Shankar Kumar and
                  William J. Byrne},
  title        = {Minimum Bayes-Risk Decoding for Statistical Machine Translation},
  booktitle    = {Proceedings of the Conference of the North American Chapter of the Association for Computational Linguistics: Human Language Technologies (NAACL)},
  pages        = {169--176},
  year         = {2004},
}

@article{shapiro1965analysis,
  title={An analysis of variance test for normality (complete samples)},
  author={Shapiro, Samuel Sanford and Wilk, Martin B},
  journal={Biometrika},
  volume={52},
  number={3-4},
  pages={591--611},
  year={1965},
  publisher={Oxford University Press}
}

@article{radford2019language,
  title={Language models are unsupervised multitask learners},
  author={Radford, Alec and Wu, Jeffrey and Child, Rewon and Luan, David and Amodei, Dario and Sutskever, Ilya and others},
  journal={OpenAI Blog},
  volume={1},
  number={8},
  pages={9},
  year={2019}
}

@inproceedings{DBLP:conf/nips/BrownMRSKDNSSAA20,
  author       = {Tom B. Brown and
                  Benjamin Mann and
                  Nick Ryder and
                  Melanie Subbiah and
                  Jared Kaplan and
                  Prafulla Dhariwal and
                  Arvind Neelakantan and
                  Pranav Shyam and
                  Girish Sastry and
                  Amanda Askell and
                  Sandhini Agarwal and
                  Ariel Herbert{-}Voss and
                  Gretchen Krueger and
                  Tom Henighan and
                  Rewon Child and
                  Aditya Ramesh and
                  Daniel M. Ziegler and
                  Jeffrey Wu and
                  Clemens Winter and
                  Christopher Hesse and
                  Mark Chen and
                  Eric Sigler and
                  Mateusz Litwin and
                  Scott Gray and
                  Benjamin Chess and
                  Jack Clark and
                  Christopher Berner and
                  Sam McCandlish and
                  Alec Radford and
                  Ilya Sutskever and
                  Dario Amodei},
  title        = {Language Models are Few-Shot Learners},
  booktitle    = {Proceedings of the Annual Conference on Neural Information Processing Systems (NeurIPS)},
  year         = {2020},
}

@inproceedings{DBLP:conf/iclr/ZhangKWWA20,
  author       = {Tianyi Zhang and
                  Varsha Kishore and
                  Felix Wu and
                  Kilian Q. Weinberger and
                  Yoav Artzi},
  title        = {BERTScore: Evaluating Text Generation with {BERT}},
  booktitle    = {Proceedings of the International Conference on Learning Representations (ICLR)},
  year         = {2020},
}

@inproceedings{DBLP:conf/sigir/ZhuLZGZWY18,
  author       = {Yaoming Zhu and
                  Sidi Lu and
                  Lei Zheng and
                  Jiaxian Guo and
                  Weinan Zhang and
                  Jun Wang and
                  Yong Yu},
  title        = {Texygen: {A} Benchmarking Platform for Text Generation Models},
  booktitle    = {Proceedings of the {ACM} {SIGIR} International Conference on Research {\&} Development in Information Retrieval (SIGIR)},
  pages        = {1097--1100},
  year         = {2018},
}

@inproceedings{DBLP:conf/eacl/SuCWVBLC21,
  author       = {Yixuan Su and
                  Deng Cai and
                  Yan Wang and
                  David Vandyke and
                  Simon Baker and
                  Piji Li and
                  Nigel Collier},
  title        = {Non-Autoregressive Text Generation with Pre-trained Language Models},
  booktitle    = {Proceedings of the Conference of the European Chapter of the Association for Computational Linguistics: Main Volume (EACL)},
  pages        = {234--243},
  year         = {2021},
}

@inproceedings{DBLP:conf/acl/GuK21,
  author       = {Jiatao Gu and
                  Xiang Kong},
  title        = {Fully Non-autoregressive Neural Machine Translation: Tricks of the
                  Trade},
  booktitle    = {Findings of ACL},
  pages        = {120--133},
  year         = {2021},
}

@inproceedings{DBLP:conf/nips/SahooASGMCRK24,
  author       = {Subham S. Sahoo and
                  Marianne Arriola and
                  Yair Schiff and
                  Aaron Gokaslan and
                  Edgar Marroquin and
                  Justin T. Chiu and
                  Alexander Rush and
                  Volodymyr Kuleshov},
  title        = {Simple and Effective Masked Diffusion Language Models},
  booktitle    = {Proceedings of the Annual Conference on Neural Information Processing Systems (NeurIPS)},
  year         = {2024},
}

@inproceedings{DBLP:conf/nips/AustinJHTB21,
  author       = {Jacob Austin and
                  Daniel D. Johnson and
                  Jonathan Ho and
                  Daniel Tarlow and
                  Rianne van den Berg},
  title        = {Structured Denoising Diffusion Models in Discrete State-Spaces},
  booktitle    = {Proceedings of the Annual Conference on Neural Information Processing Systems (NeurIPS)},
  pages        = {17981--17993},
  year         = {2021},
}

@inproceedings{DBLP:conf/naacl/ZhangW0G0QL22,
  author       = {Kexun Zhang and
                  Rui Wang and
                  Xu Tan and
                  Junliang Guo and
                  Yi Ren and
                  Tao Qin and
                  Tie{-}Yan Liu},
  title        = {A Study of Syntactic Multi-Modality in Non-Autoregressive Machine Translation},
  booktitle    = {Proceedings of the Conference of the North American Chapter of the Association for Computational Linguistics: Human Language Technologies (NAACL)},
  pages        = {1747--1757},
  year         = {2022},
}

@inproceedings{DBLP:conf/acl/RanLLZ20,
  author       = {Qiu Ran and
                  Yankai Lin and
                  Peng Li and
                  Jie Zhou},
  title        = {Learning to Recover from Multi-Modality Errors for Non-Autoregressive
                  Neural Machine Translation},
  booktitle    = {Proceedings of the Annual Meeting of the Association for Computational Linguistics (ACL)},
  pages        = {3059--3069},
  year         = {2020},
}

@inproceedings{DBLP:conf/aaai/ShaoZ0023,
  author       = {Chenze Shao and
                  Jinchao Zhang and
                  Jie Zhou and
                  Yang Feng},
  title        = {Rephrasing the Reference for Non-autoregressive Machine Translation},
  booktitle    = {Proceedings of the {AAAI} Conference on Artificial Intelligence (AAAI)},
  pages        = {13538--13546},
  year         = {2023},
}

@inproceedings{DBLP:conf/naacl/ShaoW022,
  author       = {Chenze Shao and
                  Xuanfu Wu and
                  Yang Feng},
  title        = {One Reference Is Not Enough: Diverse Distillation with Reference Selection  for Non-Autoregressive Translation},
  booktitle    = {Proceedings of the Conference of the North American Chapter of the Association for Computational Linguistics: Human Language Technologies (NAACL)},
  pages        = {3779--3791},
  year         = {2022},
}

@inproceedings{DBLP:conf/icml/HuangTZLH22,
  author       = {Fei Huang and
                  Tianhua Tao and
                  Hao Zhou and
                  Lei Li and
                  Minlie Huang},
  title        = {On the Learning of Non-Autoregressive Transformers},
  booktitle    = {Proceedings of the International Conference on Machine Learning (ICML)},
  pages        = {9356--9376},
  year         = {2022},
}

@article{DBLP:journals/corr/abs-2112-11640,
  author       = {Jiaxin Guo and
                  Minghan Wang and
                  Daimeng Wei and
                  Hengchao Shang and
                  Yuxia Wang and
                  Zongyao Li and
                  Zhengzhe Yu and
                  Zhanglin Wu and
                  Yimeng Chen and
                  Chang Su and
                  Min Zhang and
                  Lizhi Lei and
                  Shimin Tao and
                  Hao Yang},
  title        = {Self-Distillation Mixup Training for Non-autoregressive Neural Machine Translation},
  journal      = {CoRR},
  volume       = {abs/2112.11640},
  year         = {2021},
}

@article{DBLP:journals/corr/abs-2205-11162,
  author       = {Weizhen Qi and
                  Yeyun Gong and
                  Yelong Shen and
                  Jian Jiao and
                  Yu Yan and
                  Houqiang Li and
                  Ruofei Zhang and
                  Weizhu Chen and
                  Nan Duan},
  title        = {A Self-Paced Mixed Distillation Method for Non-Autoregressive Generation},
  journal      = {CoRR},
  volume       = {abs/2205.11162},
  year         = {2022},
}

@inproceedings{DBLP:conf/aaai/LiuBZH23,
  author       = {Min Liu and
                  Yu Bao and
                  Chengqi Zhao and
                  Shujian Huang},
  title        = {Selective Knowledge Distillation for Non-Autoregressive Neural Machine
                  Translation},
  booktitle    = {Proceedings of the {AAAI} Conference on Artificial Intelligence (AAAI)},
  pages        = {13246--13254},
  year         = {2023},
}

@inproceedings{DBLP:conf/iclr/SongME21,
  author       = {Jiaming Song and
                  Chenlin Meng and
                  Stefano Ermon},
  title        = {Denoising Diffusion Implicit Models},
  booktitle    = {Proceedings of the International Conference on Learning Representations (ICLR)},
  year         = {2021},
}

@inproceedings{DBLP:conf/nips/0011ZB0L022,
  author       = {Cheng Lu and
                  Yuhao Zhou and
                  Fan Bao and
                  Jianfei Chen and
                  Chongxuan Li and
                  Jun Zhu},
  title        = {DPM-Solver: {A} Fast {ODE} Solver for Diffusion Probabilistic Model Sampling in Around 10 Steps},
  booktitle    = {Proceedings of the Annual Conference on Neural Information Processing Systems (NeurIPS)},
  year         = {2022},
}

@article{DBLP:journals/ijautcomp/LuZBCLZ25,
  author       = {Cheng Lu and
                  Yuhao Zhou and
                  Fan Bao and
                  Jianfei Chen and
                  Chongxuan Li and
                  Jun Zhu},
  title        = {DPM-Solver++: Fast Solver for Guided Sampling of Diffusion Probabilistic
                  Models},
  journal      = {Mach. Intell. Res.},
  volume       = {22},
  number       = {4},
  pages        = {730--751},
  year         = {2025},
}

@inproceedings{DBLP:conf/icml/SongD0S23,
  author       = {Yang Song and
                  Prafulla Dhariwal and
                  Mark Chen and
                  Ilya Sutskever},
  title        = {Consistency Models},
  booktitle    = {Proceedings of the International Conference on Machine Learning (ICML)},
  pages        = {32211--32252},
  year         = {2023},
}

@inproceedings{DBLP:conf/nips/YeGCZGSWJLBK24,
  author       = {Jiacheng Ye and
                  Shansan Gong and
                  Liheng Chen and
                  Lin Zheng and
                  Jiahui Gao and
                  Han Shi and
                  Chuan Wu and
                  Xin Jiang and
                  Zhenguo Li and
                  Wei Bi and
                  Lingpeng Kong},
  title        = {Diffusion of Thought: Chain-of-Thought Reasoning in Diffusion Language Models},
  booktitle    = {Proceedings of the Annual Conference on Neural Information Processing Systems (NeurIPS)},
  year         = {2024},
}

@article{DBLP:journals/corr/abs-2110-14168,
  author       = {Karl Cobbe and
                  Vineet Kosaraju and
                  Mohammad Bavarian and
                  Mark Chen and
                  Heewoo Jun and
                  Lukasz Kaiser and
                  Matthias Plappert and
                  Jerry Tworek and
                  Jacob Hilton and
                  Reiichiro Nakano and
                  Christopher Hesse and
                  John Schulman},
  title        = {Training Verifiers to Solve Math Word Problems},
  journal      = {CoRR},
  volume       = {abs/2110.14168},
  year         = {2021},
}

@article{DBLP:journals/corr/abs-2507-15857,
  author       = {Mihir Prabhudesai and
                  Mengning Wu and
                  Amir Zadeh and
                  Katerina Fragkiadaki and
                  Deepak Pathak},
  title        = {Diffusion Beats Autoregressive in Data-Constrained Settings},
  journal      = {CoRR},
  volume       = {abs/2507.15857},
  year         = {2025}
}

@article{DBLP:journals/corr/abs-2502-09992,
  author       = {Shen Nie and
                  Fengqi Zhu and
                  Zebin You and
                  Xiaolu Zhang and
                  Jingyang Ou and
                  Jun Hu and
                  Jun Zhou and
                  Yankai Lin and
                  Ji{-}Rong Wen and
                  Chongxuan Li},
  title        = {Large Language Diffusion Models},
  journal      = {CoRR},
  volume       = {abs/2502.09992},
  year         = {2025},
}

@article{DBLP:journals/corr/abs-2508-15487,
  author       = {Jiacheng Ye and
                  Zhihui Xie and
                  Lin Zheng and
                  Jiahui Gao and
                  Zirui Wu and
                  Xin Jiang and
                  Zhenguo Li and
                  Lingpeng Kong},
  title        = {Dream 7B: Diffusion Large Language Models},
  journal      = {CoRR},
  volume       = {abs/2508.15487},
  year         = {2025},
}

@inproceedings{DBLP:conf/eacl/HuWAMFOS24,
  author       = {Vincent Tao Hu and
                  Di Wu and
                  Yuki Markus Asano and
                  Pascal Mettes and
                  Basura Fernando and
                  Bj{\"{o}}rn Ommer and
                  Cees Snoek},
  title        = {Flow Matching for Conditional Text Generation in a Few Sampling Steps},
  booktitle    = {Proceedings of the 18th Conference of the European Chapter of the
                  Association for Computational Linguistics, {EACL} 2024 - Volume 2:
                  Short Papers, St. Julian's, Malta, March 17-22, 2024},
  pages        = {380--392},
  publisher    = {Association for Computational Linguistics},
  year         = {2024},
}

@article{DBLP:journals/corr/abs-2506-00290,
  author       = {Tianqi Chen and
                  Shujian Zhang and
                  Mingyuan Zhou},
  title        = {DLM-One: Diffusion Language Models for One-Step Sequence Generation},
  journal      = {CoRR},
  volume       = {abs/2506.00290},
  year         = {2025},
}

@inproceedings{DBLP:conf/nips/DarasDDD23,
  author       = {Giannis Daras and
                  Yuval Dagan and
                  Alex Dimakis and
                  Constantinos Daskalakis},
  title        = {Consistent Diffusion Models: Mitigating Sampling Drift by Learning
                  to be Consistent},
  booktitle    = {Advances in Neural Information Processing Systems 36: Annual Conference
                  on Neural Information Processing Systems 2023, NeurIPS 2023, New Orleans,
                  LA, USA, December 10 - 16, 2023},
  year         = {2023},
}
\clearpage
\newpage
\appendix

\begin{algorithm}[tb]
    \caption{Training with \scheduler}
    \label{alg:train}
    \textbf{Input}: Text sequence $\boldsymbol{x}$, denoising network $D_\theta$
    \begin{algorithmic}[1]
        \While{not converged}
        \State $ \boldsymbol{z}_0 \sim q_\theta(\boldsymbol{z}_0|\boldsymbol{x})$
        \State $\ t \sim \mathcal{U}(\epsilon, 1)$, $\boldsymbol{\epsilon} \sim \mathcal{N}(0, \mathbf{I})$
        \State \textcolor{blue}{$\boldsymbol{z}_{t}' = \alpha_t\lambda_t \boldsymbol{z}_0 + \sigma_t\sqrt{1+\gamma_t^2} \boldsymbol{\epsilon}$} 
        \State \textcolor{blue}{$\hat{\boldsymbol{z}}_\theta^{t} \leftarrow D_\theta(\boldsymbol{z}_{t}',0)$}
        \State $r \sim \mathcal{U}(0,1)$
        \State $\boldsymbol{z}_\theta^{t} \leftarrow D_\theta(\boldsymbol{z}_{t}',\hat{\boldsymbol{z}}_\theta^{t})$ if $r<0.5$ else $\hat{\boldsymbol{z}}_\theta^{t}$
        \State $\mathcal{L}_{\text{total}}=||\boldsymbol{z}_\theta^{t}-\boldsymbol{z}_0||^2 - \log p_\theta(\boldsymbol{x}|\boldsymbol{z}_0)$ 
        \State $\theta\leftarrow \theta- \nabla_\theta \mathcal{L}_{\text{total}} $
        \EndWhile
    \end{algorithmic}
\end{algorithm}

\section{Theoretical Details}
\label{sec:appendix}

\subsection{Derivation Of The Posterior Distribution}
\label{subsec:posterior_dist_derv}
Since the forward process is a Markov chain, for $t>s$, we have  $q(\boldsymbol{z}_s,\boldsymbol{z}_t|\boldsymbol{z}_0)=q(\boldsymbol{z}_s|\boldsymbol{z}_0)q(\boldsymbol{z}_t|\boldsymbol{z}_s)$.
Following the Bayes rule, the posterior equals to $q(\boldsymbol{z}_t|\boldsymbol{z}_s)q(\boldsymbol{z}_s|\boldsymbol{z}_0)/q(\boldsymbol{z}_t|\boldsymbol{z}_0)$.
Given that every term is a Gaussian likelihood (Eq.~\ref{eq:forward_1} and ~\ref{eq:forward_2}), we plug this into the posterior, which yields:
\begin{align}
    q(\boldsymbol{z}_{s}|\boldsymbol{z}_t,\boldsymbol{z}_0)&= \mathcal{N}(\boldsymbol{z}_s;\tilde{\boldsymbol{\mu}}(\boldsymbol{z}_t,\boldsymbol{z}_0),\tilde{\sigma}_t^2\boldsymbol{I}) \\
    \text{where}\ \ \tilde{\boldsymbol{\mu}}(\boldsymbol{z}_t,\boldsymbol{z}_0) &=  \frac{\alpha_t}{\alpha_s}\frac{\sigma_s^2}{\sigma_t^2}\boldsymbol{z}_t+\alpha_s\frac{\sigma_{t|s}^2}{\sigma_{t}^2}\boldsymbol{z}_0 \\
    \text{and} \quad\quad\quad \ \ \
    \tilde{\sigma}_t &= \frac{\sigma_s}{\sigma_t}\sigma_{t|s}.
\end{align}

\subsection{Analysis On The Estimation Gap
}
\label{subsec:theorem_proof}
In contrast to prior studies, which suggested that error could accumulate across steps~\cite{DBLP:conf/icml/NingSPCC23}, we showed that, given a sufficiently large number of denoising steps, the discretization errors between consecutive steps become small enough for the model to estimate the correct self-condition.
Specifically, we first denote $t_{i+1}$ and $t_i$
as two selected consecutive steps used during generation, and $\boldsymbol{\bar{z}}_\theta^{t_{i}t_{i+1}}$ as the estimation with the condition is the denoising output from the previous step $t_{i+1}$. Then, we state the following theorem:
\begin{theorem}
    Let $t_0,t_1,...,t_n\in[\epsilon,1]$ such that $t_0<t_1 <...<t_n=1$; $\Delta t:=\max_{i\in[1,n-1]}\{|t_{i+1}-t_i|\}$. 
    Assume $D_\theta$ satisfies the Lipschitz condition: there exists $K>0$ such that for all $t\in[\epsilon,1]$, $\boldsymbol{x}$, and $\boldsymbol{y}$, we have $||D_\theta(\boldsymbol{z}_t,\boldsymbol{x})-D_\theta(\boldsymbol{z}_t,\boldsymbol{y})||_2\leq K||\boldsymbol{x}-\boldsymbol{y}||_2$.
    Assume further that for all $i\in[0,n-1]$, the denoising estimation at $t_{i+1}$ has local error uniformly bounded by $\mathcal{O}((t_{i+1}-t_i)^{p+1})$ with $p\geq1$. Then, the supremum of local error expectation:
    \begin{align}
        \sup &\underset{i\sim[0,n-1]}{\mathbb{E}}\left[\|D_\theta(\boldsymbol{z}_{t_i}, \boldsymbol{\bar{z}}_\theta^{t_{i+1}t_{i+2}})-D_\theta(\boldsymbol{z}_{t_i}, \boldsymbol{\hat{z}}_\theta^{t_i})\|\right] \nonumber\\& = \mathcal{O}((\Delta t)^p).
    \label{eq:multistep_bound}
    \end{align}
    \label{theorem:bound}
\end{theorem}
\textit{Proof.} Because $D_\theta(\boldsymbol{z}_{t_i},\cdot)$ is $K$-Lipschitz, we have
\begin{align*}
    &\underset{i\sim[0,n-1]}{\mathbb{E}}\|D_\theta(\boldsymbol{z}_{t_i}, \boldsymbol{\bar{z}}_\theta^{t_{i+1}t_{i+2}})-D_\theta(\boldsymbol{z}_{t_i}, \boldsymbol{\hat{z}}_\theta^{t_i})\| \\
    &\leq K\underset{i\sim[0,n-1]}{\mathbb{E}}\|\boldsymbol{\bar{z}}_\theta^{t_{i+1}t_{i+2}}-\boldsymbol{\hat{z}}_\theta^{t_i}\|
\end{align*}
Furthermore, from our assumption that the local error is bounded by $\mathcal{O}((t_{i+1}-t_i)^{p+1})$, we have the upper bound of the total local error is

\begin{align*}
    &K  \underset{i\sim[0,n-1]}{\mathbb{E}}\|\boldsymbol{\bar{z}}_\theta^{t_{i+1}t_{i+2}}-\boldsymbol{\hat{z}}_\theta^{t_i}\| \\ &\overset{(i)}{\leq} \frac{K}{n}\cdot \sum_{i=0}^{n-1}\mathcal{O}((t_{i+1}-t_i)^{p+1}) \\
    &\leq \sum_{i=0}^{n-1} \mathcal{O}((t_{i+1}-t_i)^{p+1}) \\
    &= \sum_{i=0}^{n-1} (t_{i+1}-t_i)\mathcal{O}((t_{i+1}-t_i)^{p}) \\
    &\leq \mathcal{O}((\Delta t)^p) \sum_{i=0}^{n-1} (t_{i+1}-t_i) \\
    &= \mathcal{O}((\Delta t)^p) (t_n-t_0) \\
    &\leq \mathcal{O}((\Delta t)^p) (1-\epsilon) \\
    &\leq \mathcal{O}((\Delta t)^p)
\end{align*}
where $(i)$ holds due to the uniform sampling distribution of $t$.
Our proof builds on the error bounds for the ODE Solver in Consistency Models~\cite{DBLP:conf/icml/SongD0S23}, but it is different since we provide the total local error bounds rather than targeting the empirical approximation bounds.

This theorem suggests that the self-condition estimation can become arbitrarily accurate, as long as the number of sampling steps is large enough.
In Tab.~\ref{tab:est_err}, we practically demonstrate this estimation error through different numbers of denoising steps.
\begin{table}[h!]
\centering
\small
\setlength{\tabcolsep}{5.7pt} 
\begin{tabular}{lcccc}
\toprule
\textbf{NFEs} & 5 & 20 & 50 & 100  \\
\midrule
Sup $\mathbb{E}$ & 0.047 & 0.008 & 0.009 & 0.010  \\
\bottomrule
\end{tabular}
\caption{Estimation Error.}
\label{tab:est_err}
\end{table}

\section{Self-conditioning Error Formulation}
\label{sec:sc_gauss}
This section provides empirical and analytical support for \scheduler. We show that the perturbed forward sample $\boldsymbol{z}'_t$ used in training induces self-conditioning statistics that closely match those observed at inference, and we motivate the linear parameterization of $\lambda_t$ and $\gamma_t$.

\begin{figure}[t!]
    \centering
    \begin{subfigure}{0.8\linewidth}
        \centering
        \includegraphics[width=\linewidth]{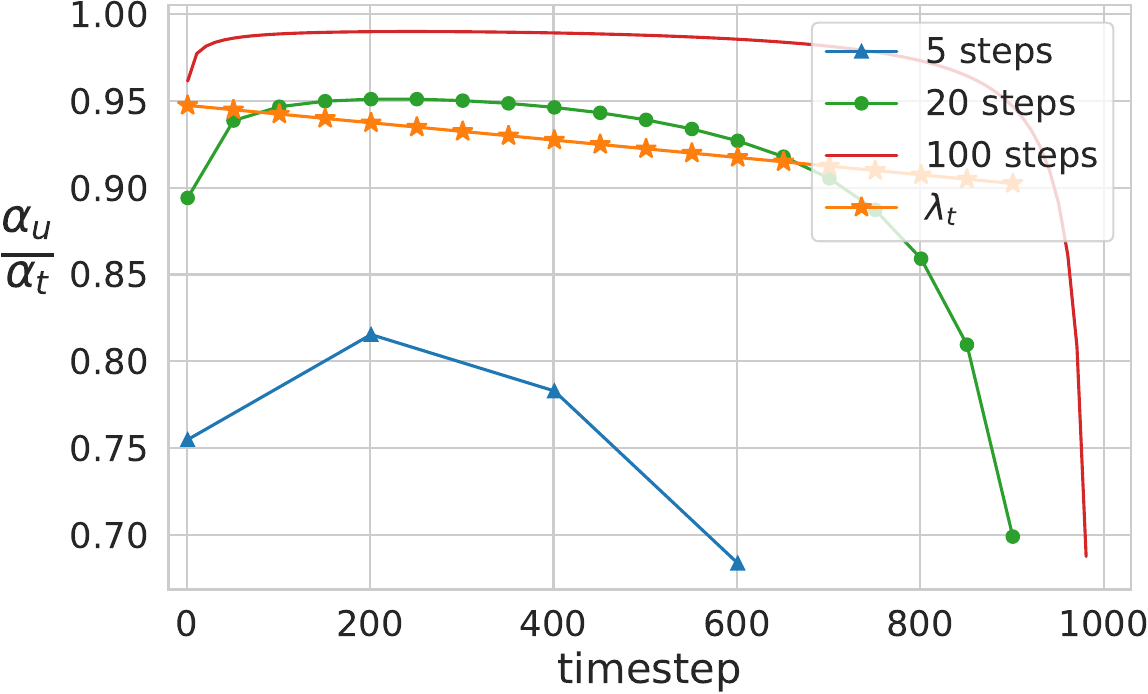}
    \end{subfigure}
    \vfill
    \begin{subfigure}{0.8\linewidth} 
    \centering
    \begin{overpic}[width=\linewidth]{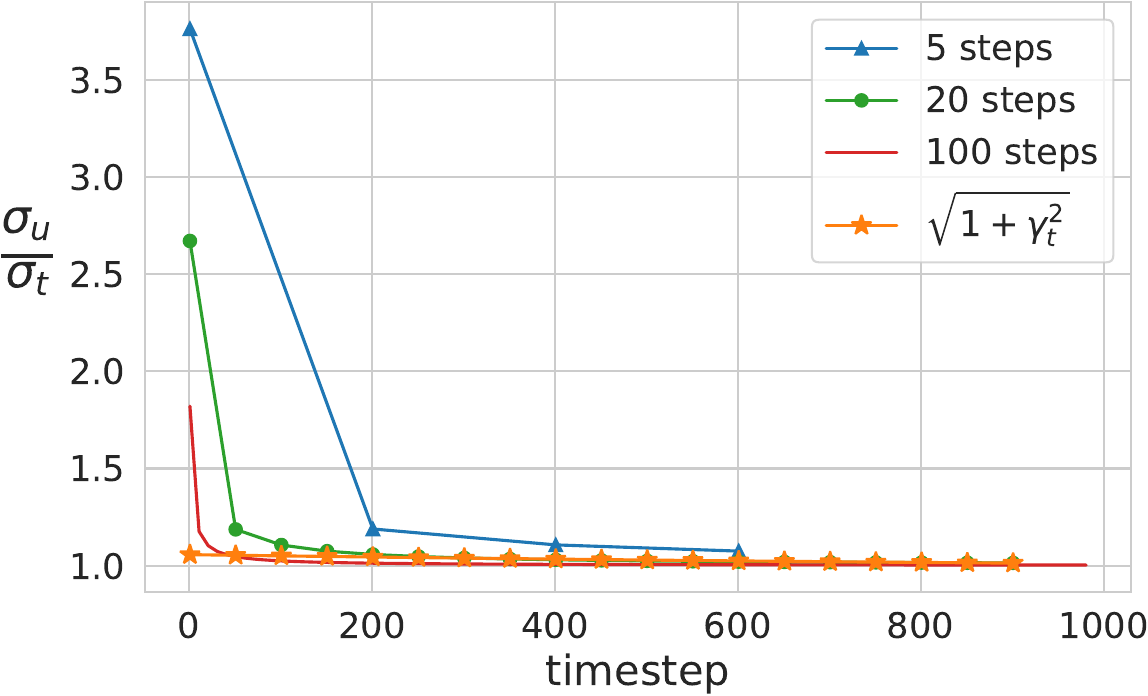}
        \put(39,23){\includegraphics[width=0.3\linewidth]{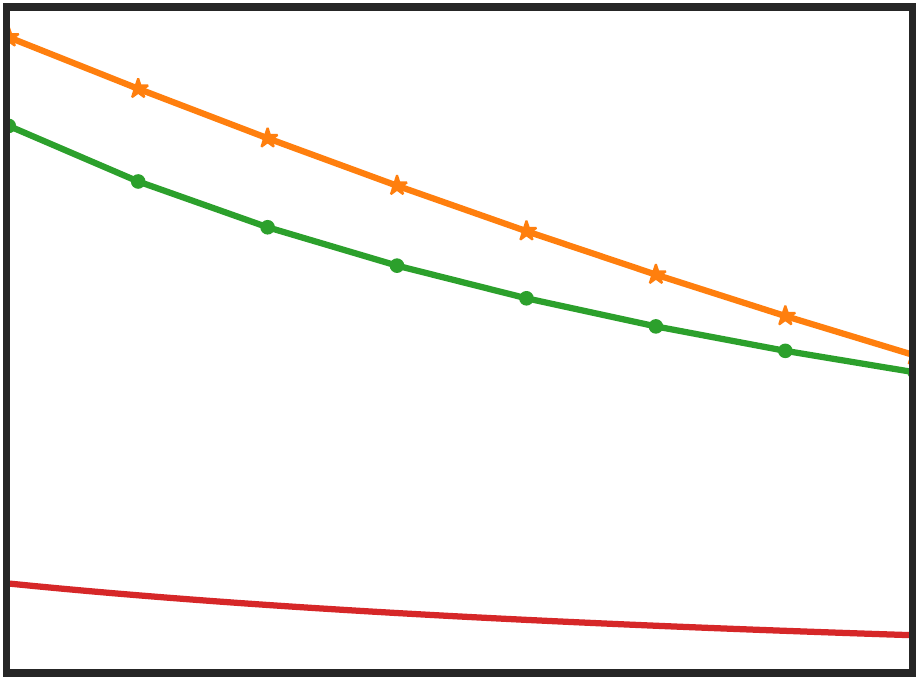}}
        \begin{tikzpicture}[remember picture, overlay]
            \draw[thin] (3.525,0.752) -- (2.42,1.41); 
            \draw[thin] (4.93,0.752) -- (4.237,1.412); 
        \end{tikzpicture}

        \put(29,25){\scriptsize 1.006}
        \put(29,39){\scriptsize 1.031}

    \end{overpic}
    \label{subfig:var_sched}
    \end{subfigure}
    \caption{Mean (top) and variance (bottom) ratio differences measured at different denoising steps. 
    We manually select hyperparameters of the scaling terms to fit 20-step sampling.
    }
    \label{fig:sched_ratio}
\end{figure}
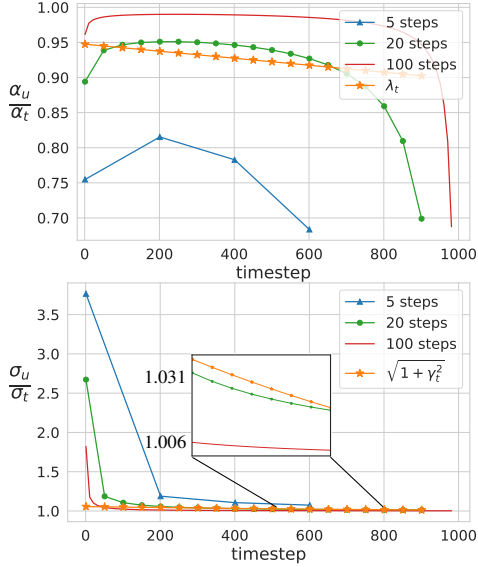

\subsection{Approximate Gaussian Distribution Between Consecutive Estimates}
Following Sec.~\ref{subsec:selfcond_limit}, we train a standard continuous diffusion language model on \textbf{IWSLT14} De-En and run the original $20$-step denoising procedure on the validation set. 
At each step, we store the previous-step reused estimate $\bar{\boldsymbol{z}}_\theta^{tu}$ and compute the corrected self-conditioning estimate for the current step, denoted $\hat{\boldsymbol{z}}_\theta^s$.

Empirically, we observe that $\bar{\boldsymbol{z}}_\theta^{tu}$ is well-approximated by a Gaussian perturbation around $\hat{\boldsymbol{z}}_\theta^s$:
\begin{align}
    \bar{\boldsymbol{z}}_\theta^{tu} \sim \mathcal{N}(\hat{\boldsymbol{\mu}}_{st}\hat{\boldsymbol{z}}_\theta^s,\hat{\boldsymbol{\sigma}}_{st}^2 \boldsymbol{I}), \label{eq:emp_gap_dist}
\end{align}
where $\hat{\boldsymbol{\mu}}_{st}$ and $\hat{\boldsymbol{\sigma}}_{st}$ are dimension-wise mean and standard deviation vectors (details in Appx.~\ref{subsec:exp_err_gauss}).
This relation makes the inference-time self-conditioning mismatch explicit and allows us to connect it to the perturbed forward construction in Eq.~\ref{eq:perturb_forward}.

Assuming that $\hat{\boldsymbol{z}}_\theta^s$ perfectly denoises $\bar{\boldsymbol{z}}_s$ at step $s$, we start from the standard forward parameterization,
\begin{align}
    \bar{\boldsymbol{z}}_s = \alpha_s \hat{\boldsymbol{z}}_\theta^s + \sigma_s \boldsymbol{\epsilon}_s,
\end{align}
and substitute $\hat{\boldsymbol{z}}_\theta^s$ using Eq.~\ref{eq:emp_gap_dist}. Rearranging terms yields
\begin{align}
    \bar{\boldsymbol{z}}_s
    &= \alpha_s\frac{\bar{\boldsymbol{z}}_\theta^{tu}-\hat{\boldsymbol{\sigma}}_{st}\boldsymbol{\epsilon_u}}{\hat{\boldsymbol{\mu}}_{st}} + \sigma_{s} \boldsymbol{\epsilon}_s \nonumber \\
    &= \alpha_s\frac{1}{\hat{\boldsymbol{\mu}}_{st}}\bar{\boldsymbol{z}}_\theta^{tu}
       + \sigma_s\sqrt{1 + \left(\frac{\alpha_{s}}{\sigma_s}\frac{\hat{\boldsymbol{\sigma}}_{st}}{\hat{\boldsymbol{\mu}}_{st}}\right)^2} \boldsymbol{\epsilon} \nonumber \\
    &= \alpha_s \hat{\boldsymbol{\lambda}}_{st}\bar{\boldsymbol{z}}_\theta^{tu} + \sigma_s \sqrt{1 + \hat{\boldsymbol{\gamma}}_{st}^2} \boldsymbol{\epsilon},
    \label{eq:leakage}
\end{align}
where $\hat{\boldsymbol{\lambda}}_{st}=1/\hat{\boldsymbol{\mu}}_{st}$, $\hat{\boldsymbol{\gamma}}_{st}=(\alpha_s/\sigma_s)(\hat{\boldsymbol{\sigma}}_{st}/\hat{\boldsymbol{\mu}}_{st})$, and $\boldsymbol{\epsilon}\sim \mathcal{N}(0,\mathbf{I})$ is obtained via reparameterization of the combined noise terms. Since $\hat{\boldsymbol{\mu}}_{st}$ and $\hat{\boldsymbol{\sigma}}_{st}$ vary independently across dimensions, all operations are element-wise.

Eq.~\ref{eq:leakage} mirrors the structure of our perturbed forward process (Eq.~\ref{eq:perturb_forward}): both corrupt the signal term and inflate the effective noise through a multiplicative factor. Consistent with this connection, the empirically observed scaling patterns in Fig.~\ref{fig:sched_infer} match the behavior induced by \scheduler in Fig.~\ref{fig:sched_ratio}, supporting the view that \scheduler exposes the model during training to the noise pattern encountered when reusing self-conditioning at inference.

\begin{figure}[t!]
    \centering
    \begin{subfigure}{0.8\linewidth}
        \centering
        \includegraphics[width=\linewidth]{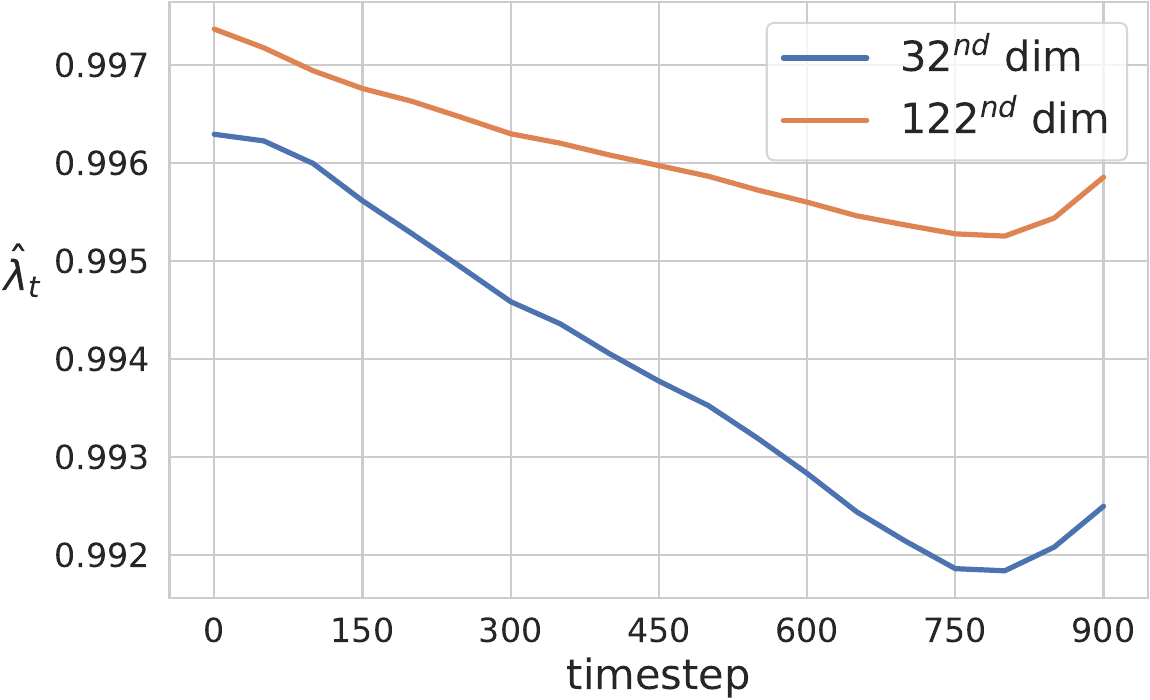}
    \end{subfigure}
    \vfill
    \begin{subfigure}{0.8\linewidth} 
        \centering
        \includegraphics[width=\linewidth]{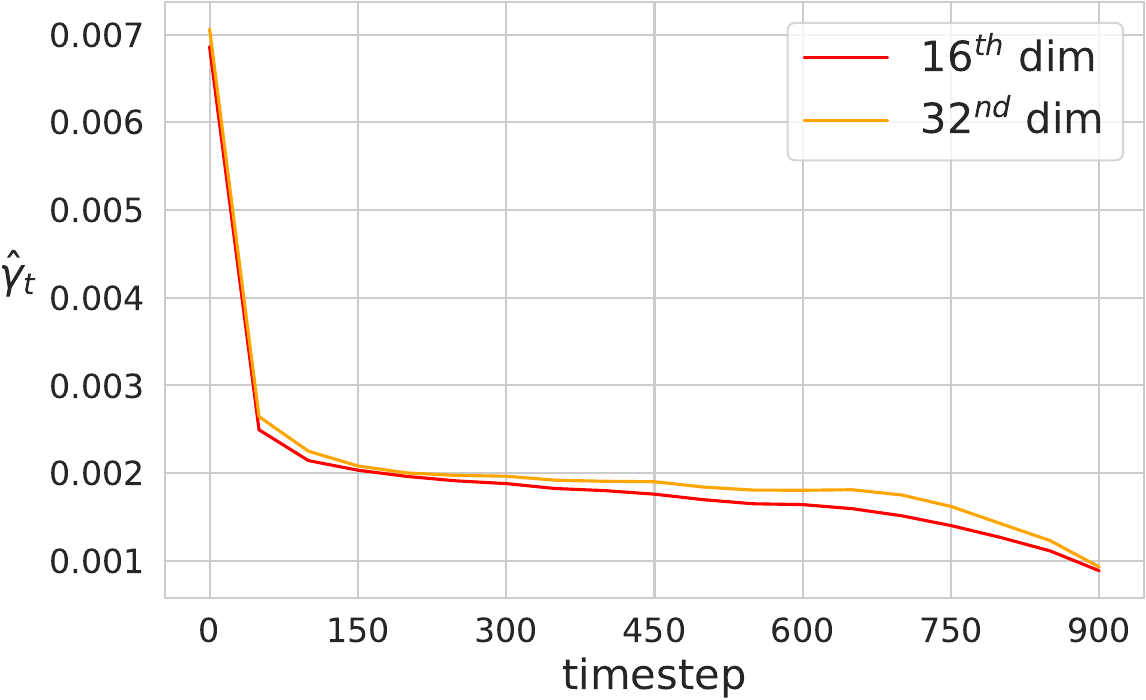}
    \end{subfigure}
    \caption{Inference mean $\hat{\lambda}_t$ (top) and variance $\hat{\gamma}_t$ (bottom) scaling factors of prediction mismatch in a pre-trained network, plotted across timestep $t$ for randomly selected embedding dimensions.}
    \label{fig:sched_infer}
\end{figure}

\begin{figure*}[htbp]
    \centering
    \begin{subfigure}{0.32\textwidth}
        \centering
        \includegraphics[width=\linewidth]{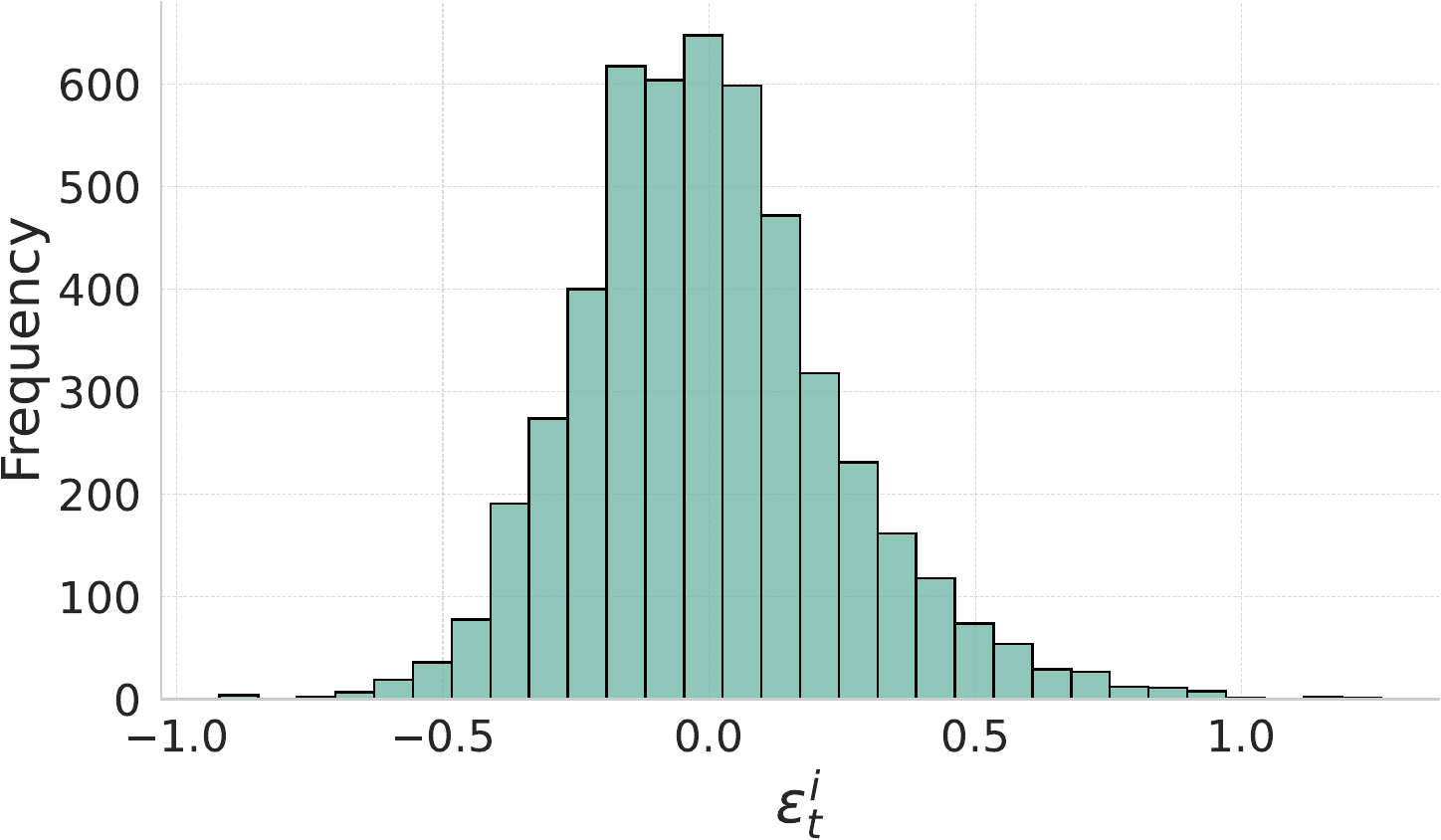}
    \end{subfigure}
    \hfill
    \begin{subfigure}{0.32\textwidth}
        \centering
        \includegraphics[width=\linewidth]{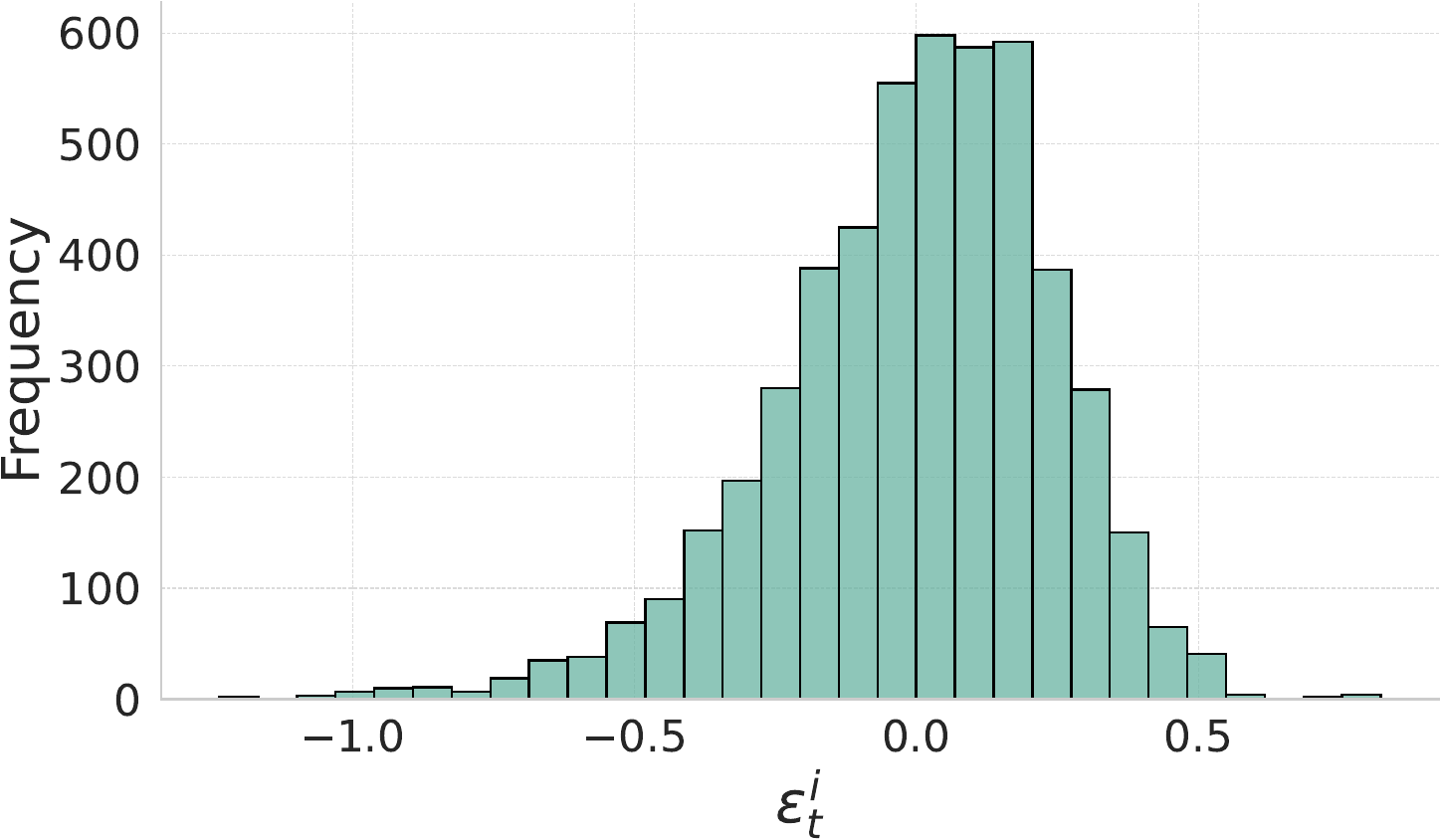}
    \end{subfigure}
    \hfill
    \begin{subfigure}{0.32\textwidth}
        \centering
        \includegraphics[width=\linewidth]{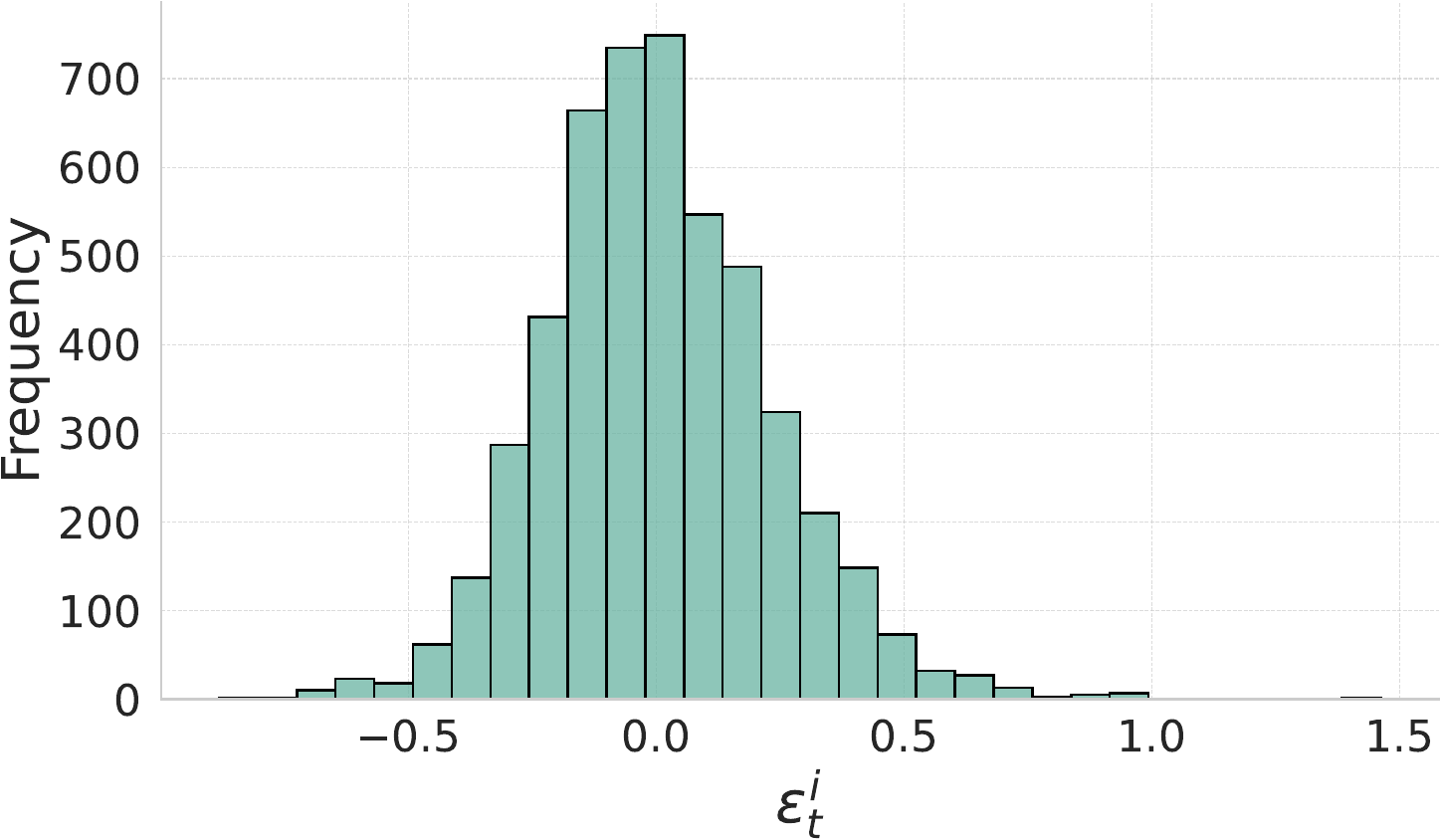}
    \end{subfigure}
    \vfill
    \begin{subfigure}{0.32\textwidth}
        \centering
        \includegraphics[width=\linewidth]{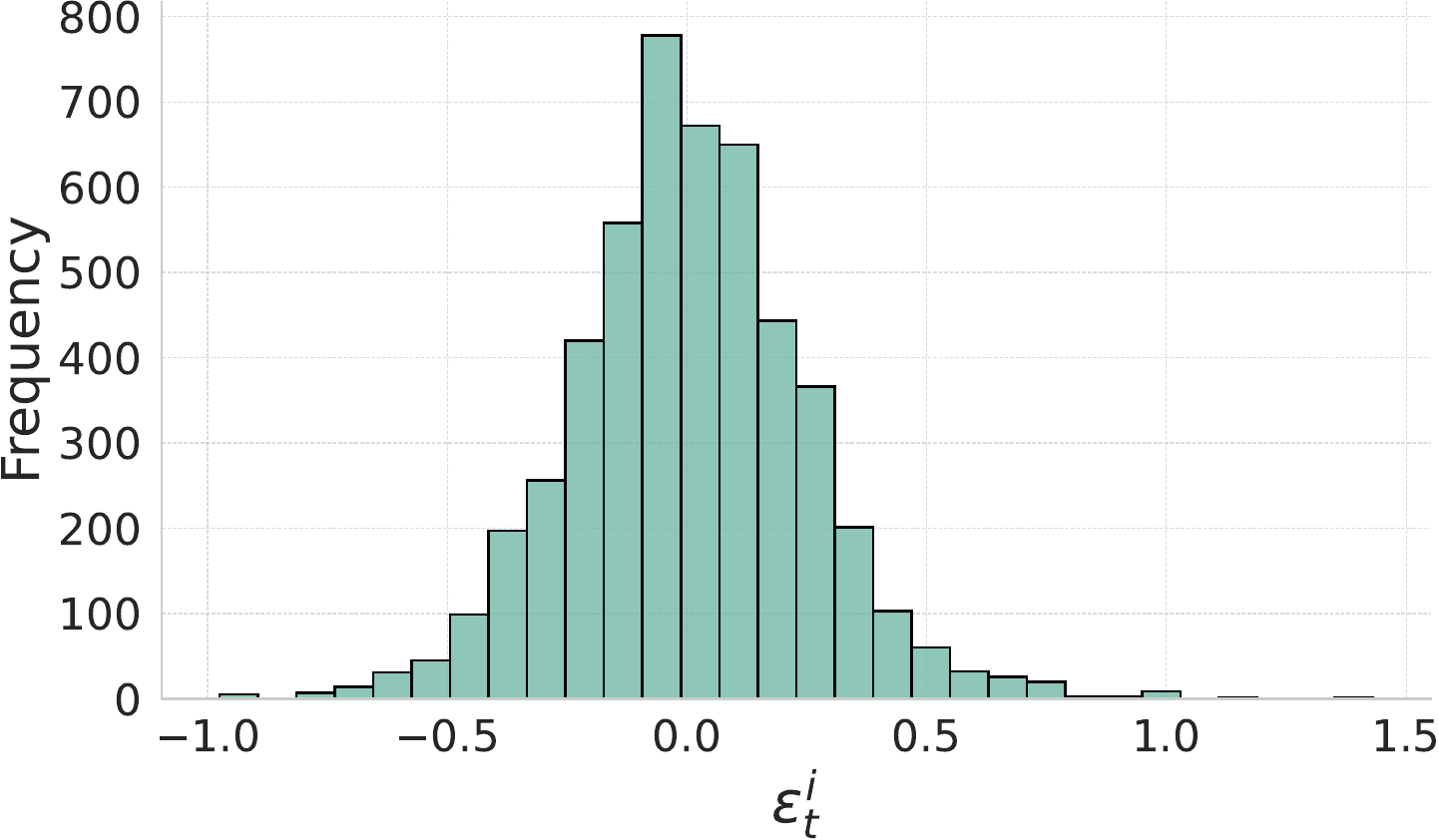}
    \end{subfigure}
    \hfill
    \begin{subfigure}{0.32\textwidth}
        \centering
        \includegraphics[width=\linewidth]{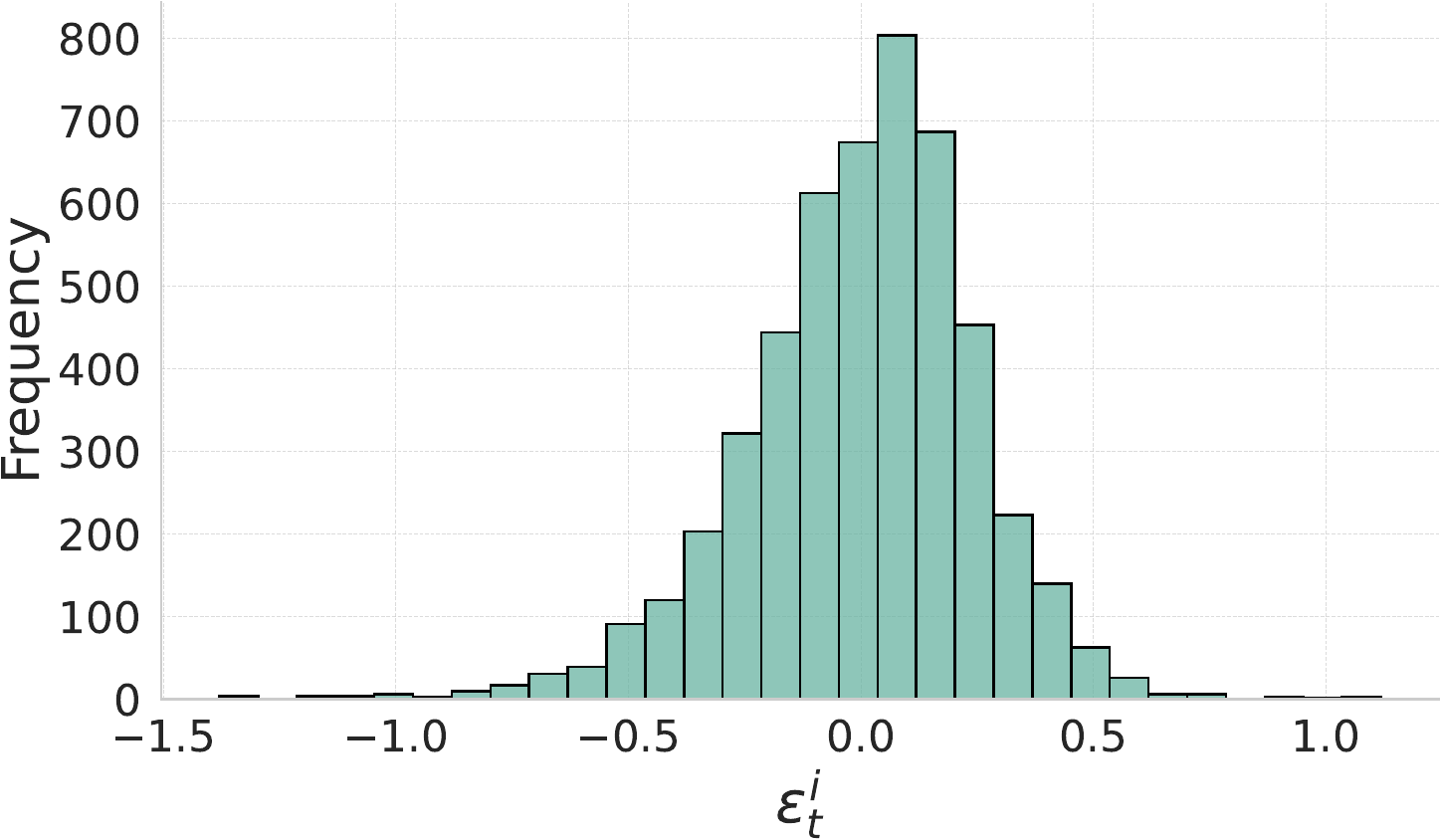}
    \end{subfigure}
    \hfill
    \begin{subfigure}{0.32\textwidth}
        \centering
        \includegraphics[width=\linewidth]{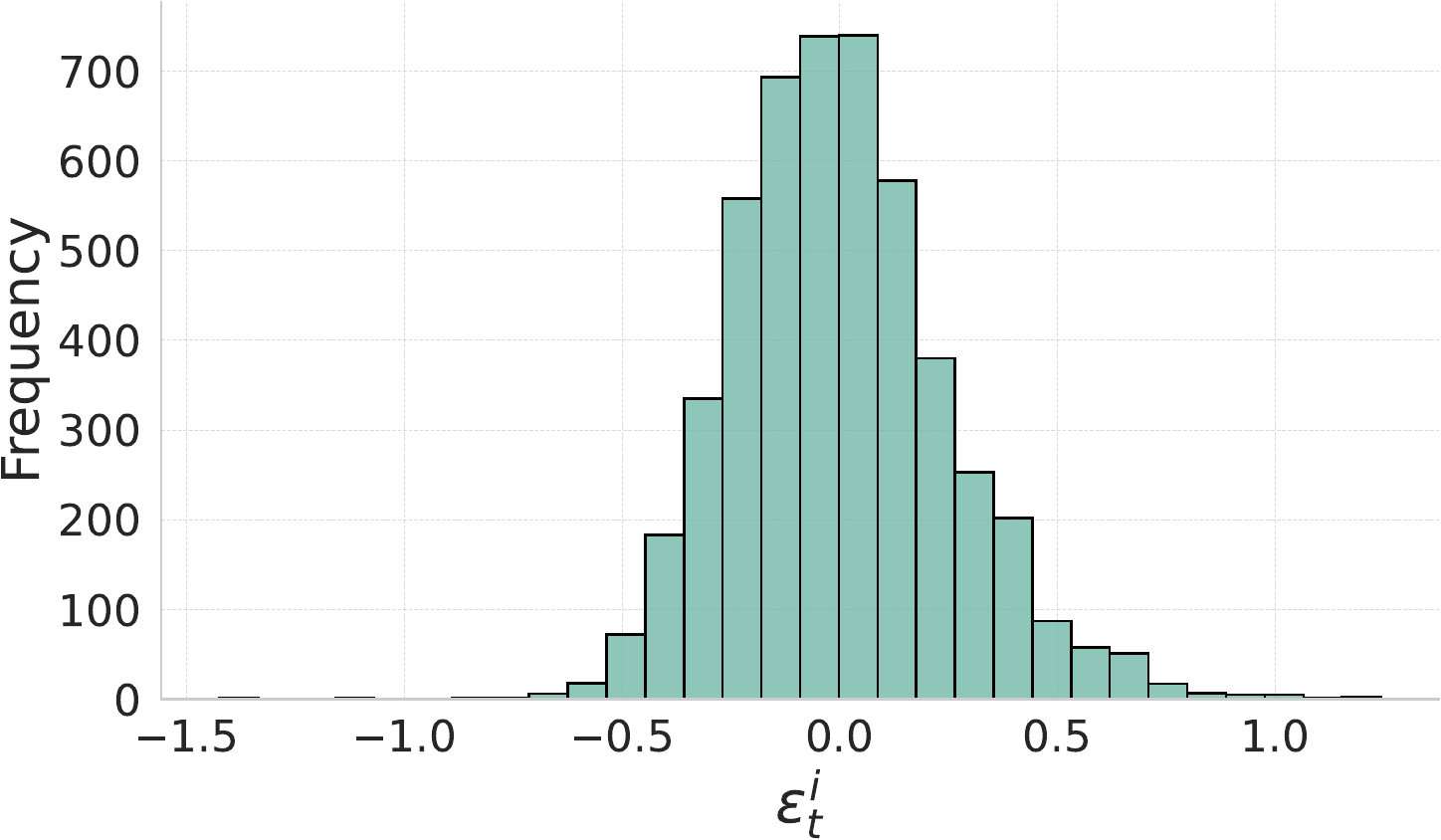}
    \end{subfigure}
    \vfill
    \begin{subfigure}{0.32\textwidth}
        \centering
        \includegraphics[width=\linewidth]{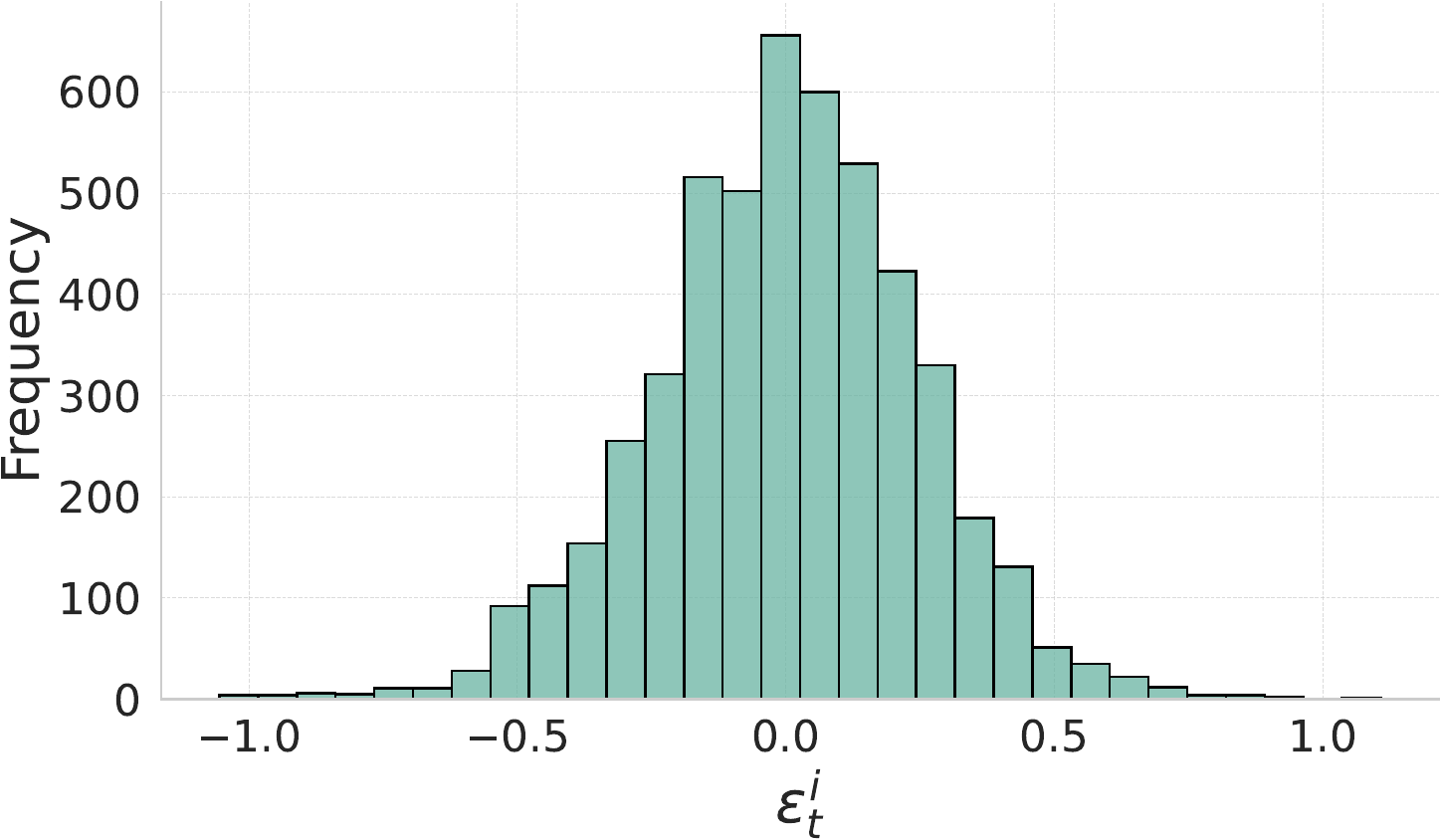}
    \end{subfigure}
    \hfill
    \begin{subfigure}{0.32\textwidth}
        \centering
        \includegraphics[width=\linewidth]{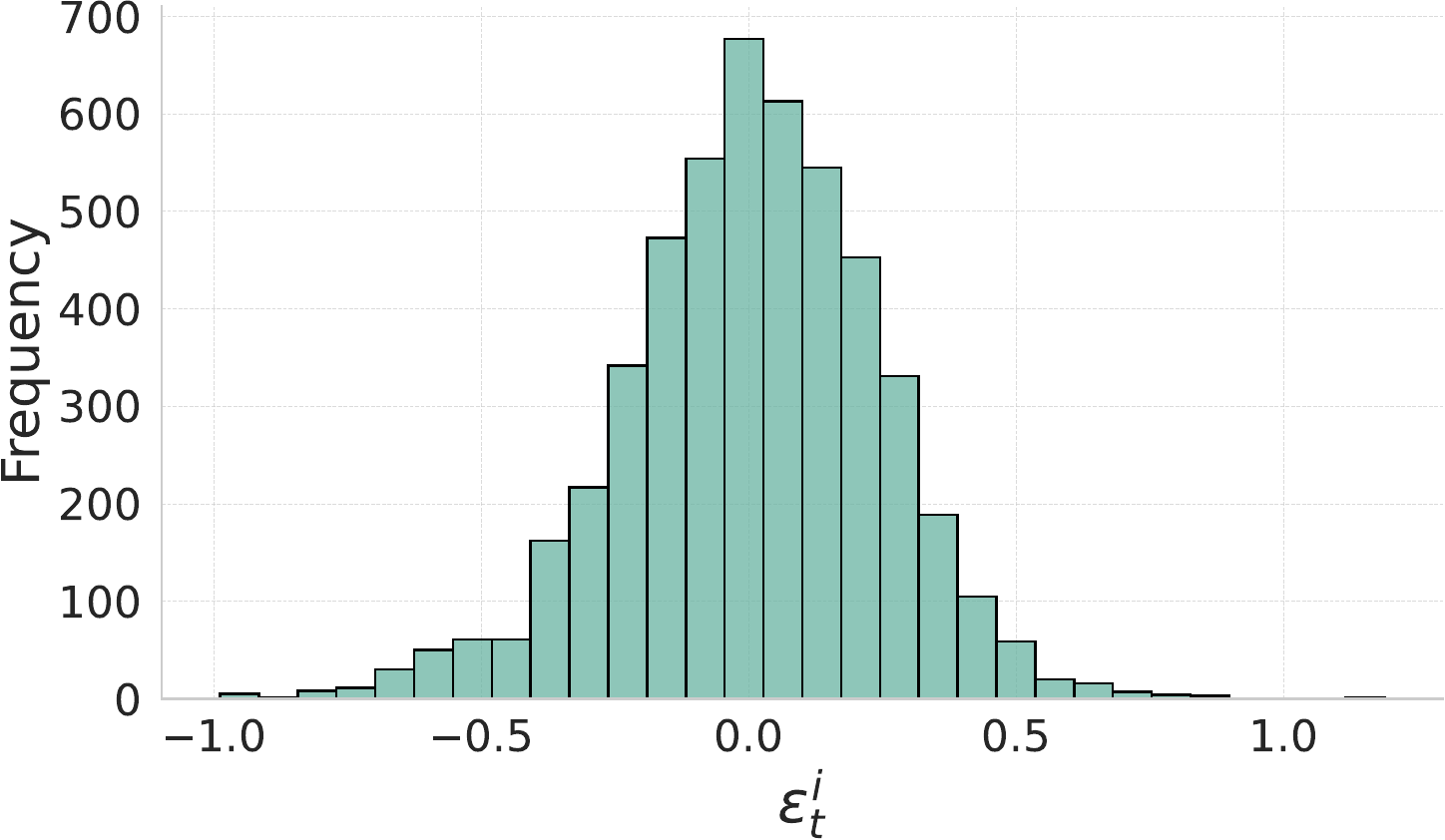}
    \end{subfigure}
    \hfill
    \begin{subfigure}{0.32\textwidth}
        \centering
        \includegraphics[width=\linewidth]{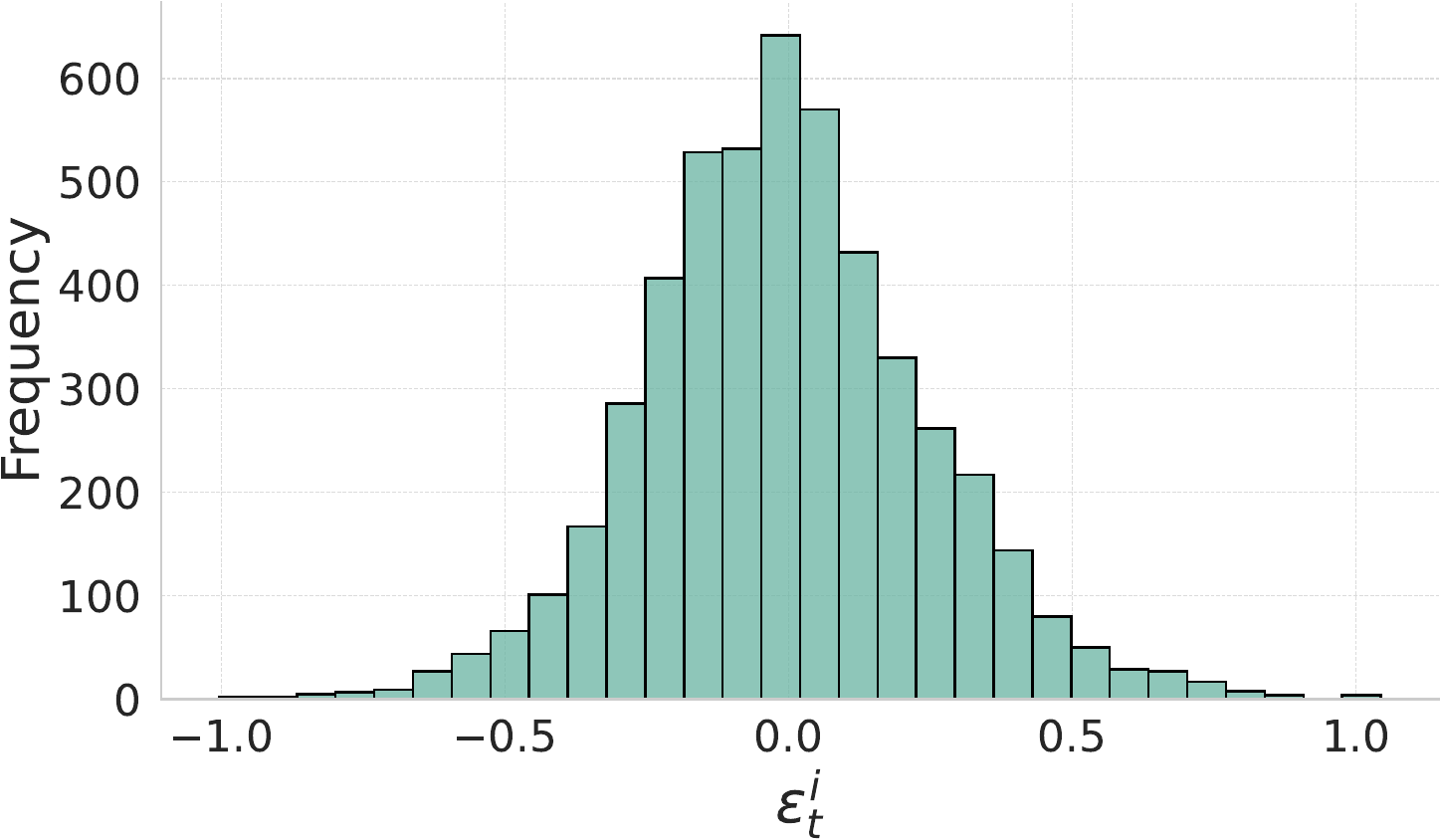}
    \end{subfigure}
    \caption{The empirical distribution of $\epsilon_t^i$ with different random values of $t$ and $i$.}
    \label{fig:empirical-distribution}
\end{figure*}

In principle, the empirical schedules $\{\hat{\boldsymbol{\lambda}}_{st},\hat{\boldsymbol{\gamma}}_{st}\}_{t=1}^T$ could be estimated directly from Eq.~\ref{eq:leakage}. 
However, these quantities depend on the discretization size and vary across dimensions, making them difficult to express with a single global function.
To avoid high-dimensional hyperparameter search, we instead use feature-independent anchor points $(\lambda_{\min}, \lambda_{\max})$ and $(\gamma_{\min}, \gamma_{\max})$ to define monotone linear schedules.
This yields a simple end-to-end procedure that preserves the intended dynamics without introducing additional preprocessing overhead.

\subsection{Estimate The Self-condition Gaussian Distribution Error}
\label{subsec:exp_err_gauss}

We now describe how we estimate the statistics in Eq.~\ref{eq:emp_gap_dist} on \textbf{IWSLT14} De-En.
For each diffusion time $t$ and each embedding dimension $i\in\{1,\ldots,H\}$, we consider the dimension-wise residual
\begin{align}
    \epsilon_t^i = \bar{z}_{\theta}^{tu,i} - \hat{\mu}_{t}^i \hat{z}_{\theta}^{s,i},
\end{align}
and test whether $\epsilon_t^i$ is approximately Gaussian with variance $(\hat{\sigma}_t^i)^2$.

We uniformly select 20 timesteps in $[\epsilon,1]$, run denoising at these timesteps, and record paired values of $\bar{\boldsymbol{z}}_\theta^{tu}$ and $\hat{\boldsymbol{z}}_\theta^s$.
For each selected $t$, we sample 1{,}000 sentences $z\in\mathcal{D}$, flatten all tokens, and estimate $\hat{\mu}_t^i$ and $\hat{\sigma}_t^i$ for each dimension.

Estimating $\hat{\mu}_t^i$ reduces to a one-dimensional linear regression (OLS) per dimension:
\begin{align}
    \hat{\mu}_t^i=\frac{\sum_{z\in\mathcal{D}} \bar{z}_u^i \hat{z}_t^i}{\sum_{z\in\mathcal{D}}(\hat{z}_t^i)^2},
\end{align}
and $\hat{\sigma}_t^i$ is computed as the standard deviation of the residuals:
\begin{align}
    \hat{\sigma}_t^i = \sqrt{\frac{1}{|\mathcal{D}|}\sum_{z\in \mathcal{D}}(\bar{z}_u^i - \hat{\mu}_t^i \hat{z}_t^i)^2 }.
\end{align}

We then standardize residuals and apply the Shapiro-Wilk test~\cite{shapiro1965analysis} to 50 randomly selected standardized residuals per dimension. Using a 95\% confidence level, we reject normality if $p<0.05$. The null hypothesis is rejected only in a small minority of cases, supporting our Gaussian approximation. Fig.~\ref{fig:empirical-distribution} shows histograms of $\epsilon_t^i$ across different dimensions.

\section{Comparison with Other Methods}
\label{sec:comparison}

\begin{table}[t]
    \centering
    \small
    \setlength{\tabcolsep}{2.2pt} 
    \begin{tabular}{l | c c | c c c}
        \toprule
        \multirow{2}{*}{\textbf{Method}} & \multirow{2}{*}{MBR} & \multirow{2}{*}{NFE} & \textit{Rouge-1} & \textit{Rouge-2} & \textit{Rouge-L} \\
        & & & $(\uparrow)$ & $(\uparrow)$ & $(\uparrow)$ \\
        \midrule
        \texttt{LSTM} & \multicolumn{1}{c}{5} & - & 34.2 & 16.0 & 31.8 \\
        \midrule
        \texttt{CMLM} & - & - & 34.4 & 15.6 & 32.2 \\
        \texttt{NAG-BERT} & - & - & \underline{35.1} & 16.5 & \textbf{33.3} \\
        \midrule 
        \texttt{Difformer}* & 10 & 20 & 34.9 & \underline{17.0} & 32.4 \\
        \texttt{DiffuSeq} & 10 & 2000 & 31.2 & 12.2 & 29.2 \\
        \texttt{SeqDiffuSeq} & 0 & 2000 & 31.9 & 12.4 & 29.2 \\
        \midrule
        \framework & 10 & 5 & 34.9 & 16.9 & 32.5 \\
        \framework & 10 & 20 & \textbf{35.3} & \textbf{17.3} & \underline{32.8} \\
        \bottomrule
    \end{tabular}
    \caption{Main results on \textbf{Gigaword}. 
    The best \texttt{NAR} results are \textbf{bold} and the second-best results are \underline{underlined}. Baseline results are from \texttt{Difformer}~\cite{DBLP:conf/naacl/GaoG0ZZ0X24}. $*$ indicates reproduced results.}
    \label{tab:gigaword_rouge}
\end{table}

\subsection{Training And Inference Mismatch}
\texttt{Distance Penalty}~\cite{DBLP:conf/acl/TangWZLCZ23} is proposed to perturb the forward process for train-test discrepancy reduction.
While conceptually related, it does not target the \emph{self-conditioning} mismatch that is central in our analysis. In addition, this strategy applies a fixed penalty across steps, which mirrors the fixed variants in Tab.~\ref{tab:abl_sched_var} and is empirically less effective than our time-varying perturbation.

\texttt{TREC} also discusses collapse phenomena under self-conditioning, however their explanation emphasizes shortcut behavior induced by a predicted prior rather than the discretization-amplified self-conditioning mismatch that we studied. Our approach directly regularizes the self-condition so that training better matches inference-time reuse errors under few-step updates.

\subsection{Modified Noise Scheduler}
\texttt{SeqDiffuSeq}~\cite{DBLP:conf/naacl/YuanYTHH24} assigns token difficulty by \emph{position}, yielding position-specific training trajectories. However, the far later position schedules become similar (e.g., their Fig.~2), suggesting diminishing
improvement as sequence length grows. In contrast, \noisescaling is length-agnostic: it adjusts noise based on model confidence at the token level, and therefore independent of the sequence length.

\texttt{DINOISER}~\cite{DBLP:journals/corr/abs-2302-10025} emphasizes high-intensity noise to strengthen learning, but high-noise training can introduce large variance and bias learning toward marginal distributions, eventually reducing sample diversity. 
Our \noisescaling instead increases noise \emph{selectively} for high-confidence tokens while leaving uncertain tokens unchanged, maintaining meaningful supervision.

\texttt{Meta-Diffu$B$}~\cite{DBLP:conf/nips/ChuangHLGLCL24} learns a planning schedule using an auxiliary controller (\texttt{LSTM}) optimized with reinforcement learning. This adds extra modules and increases training complexity, which can be unstable in practice, and its generalization to unseen contexts or long sequences depend largely on the planner choice. In contrast, \framework avoids training an auxiliary planner, and are simple yet effective.

\section{Additional Results}
\label{sec:add_result}
Tab.~\ref{tab:gigaword_rouge} shows the experimental results on Text Summarization benchmark. 
Consistent with previous results in Tabs.~\ref{tab:iwslt14_bleu} and~\ref{tab:qqp_and_wa_results}, we observe that our model outperforms most of the diffusion and non-autoregressive baselines on all metrics.

\subsection{SCP Sensitivity Analysis.}
We analyze SCP over a broad range of anchor points $(\lambda_{\min}, \lambda_{\max}, \gamma_{\min}, \gamma_{\max})$, as shown in Table~\ref{tab:scp_sensitivity}. Empirically, SCP consistently improves BLEU across all tested settings, indicating that the method is robust and not overly sensitive to the exact choice of anchor values.

\begin{table*}[t]
\centering
\setlength{\tabcolsep}{4pt}
\begin{tabular}{llcccccccccc}
\toprule
Param & Metric & Default & 1 & 2 & 3 & 4 & 5 & 6 & 7 & 8 & 9 \\
\midrule
\midrule
$\lambda_{\min}$ 
& Value & -- & 0.85 & 0.86 & 0.87 & 0.88 & 0.89 & 0.90 & 0.91 & 0.92 & 0.93  \\
& BLEU  & 26.94 & 28.67 & 28.58 & \textbf{28.73} & 28.65 & 28.68 & 28.67 & 28.34 & 28.11 & 28.44 \\
\midrule
$\lambda_{\max}$ 
& Value & -- & 0.91 & 0.92 & 0.93 & 0.94 & 0.95 & 0.96 & 0.97 & 0.98 & 0.99 \\
& BLEU  & 26.94 & 28.54 & 28.45 & \textbf{28.68} & 28.53 & 28.60 & 28.59 & 28.45 & 28.66 & 28.59 \\
\midrule
$\gamma_{\min}$ 
& Value & -- & 0.11 & 0.12 & 0.13 & 0.14 & 0.15 & 0.16 & 0.17 & 0.18 & 0.19 \\
& BLEU  & 26.94 & 28.60 & 28.37 & 28.00 & 28.46 & 28.53 & 28.48 & 28.40 & \textbf{28.79} & 28.52 \\
\midrule
$\gamma_{\max}$ 
& Value & -- & 0.31 & 0.32 & 0.33 & 0.34 & 0.35 & 0.36 & 0.37 & 0.38 & 0.39 \\
& BLEU  & 26.94 & 28.22 & 28.14 & 28.41 & \textbf{28.82} & 28.79 & 28.66 & 28.64 & 28.52 & 28.43 \\
\bottomrule
\end{tabular}%
\caption{BLEU on IWSLT14 baseline under different SCP schedule parameters. Default denotes $\lambda_t = 1$ and $\gamma_t = 0$.}
\label{tab:scp_sensitivity}
\end{table*}

Finally, the MDLM experiment is a direct application of the SCP parameter selection strategy. The results in Tab.~\ref{tab:mdlm_result} further confirm that SCP outperforms standard self-conditioning.

\section{Experimental Settings}
\label{sec:exp_conf}

\begin{table*}[t!]
\centering
\small
\setlength{\tabcolsep}{0.8pt}
\begin{tabular}{lccccccc}
\toprule
\textbf{Configurations} & \textbf{WMT14} & \textbf{WMT16} & \textbf{IWSLT14} & \textbf{Gigaword} & \textbf{QQP} & \textbf{Wiki-Auto} \\
\midrule
\textbf{Split} \\
Training    & 4,500,966 & 608,319   & 160,215 & 3,803,957 & 144,715 & 677,751 \\
Validation  & 3,000     & 1,999     & 7,282   & 189,651   & 2,048   & 2,048 \\
Test        & 3,003     & 1,999     & 6,750   & 1,951     & 2,500   & 5,000 \\
\midrule
\textbf{Preprocess} \\
BPE & 40,000 & 30,000 & 10,000 & 60,000 & 15,000 & 40,000 \\
Vocab & 40,624 & 34,976 & 15,480 & 56,392 & 15,136 & 45,376 \\
\midrule
\textbf{Architecture} \\
$d_{\text{model}}$ & 512 & 512 & 512 & 512 & 768 & 768 \\
$d_{\text{ffn}}$ & 2048 & 2048 & 1024 & 2048 & 3072 & 3072 \\
Heads & 8 & 8 & 4 & 8 & 12 & 12 \\
\midrule
\textbf{Training} \\
GPUs & 2 & 2 & 2 & 2 & 2 & 4 \\
Steps & 600K & 150K & 300K & 300K & 50K & 30K \\
Tokens/GPU & 32K & 32K & 4K & 32K & 4K & 8K \\
Phase & [100K,200K,600K] & [50K,100K,150K] & [100K,200K,300K] & [100K,200K,300K] & [10K,20K,30K] & [5K,10K,30K]  \\
Scaling & [2.0,3.0,4.0] & [2.0,3.0,4.0] & [2.0,3.0,4.0] & [2.0,3.0,4.0] & [2.0,3.0,4.0] & [2.0,4.0,8.0]   \\
\bottomrule
\end{tabular}
\caption{The dataset details, model architectures, and hyperparameters used in our experiments.}
\label{tab:hyperparams}
\end{table*}

\paragraph{Data.}
For preprocessing, we use \texttt{fairseq} library for \textbf{IWSLT14}, and use the preprocessed data released by Fully-NAT~\cite{DBLP:conf/acl/GuK21} for \textbf{WMT14} and \textbf{WMT16}\footnote{https://github.com/shawnkx/Fully-NAT}.
For \textbf{Wiki} and \textbf{QQP}, we use the ones from \texttt{DiffuSeq}\footnote{https://github.com/Shark-NLP/DiffuSeq}, and for \textbf{Gigaword} we use the \texttt{HuggingFace} version\footnote{https://huggingface.co/datasets/Harvard/gigaword}.
All datasets are tokenized with byte-pair encoding (BPE) and processed with \texttt{fairseq-preprocess}. BPE settings and vocabulary sizes are reported in Tab.~\ref{tab:hyperparams}.

\paragraph{Model.}
We use a \texttt{Transformer}-base backbone~\cite{DBLP:conf/nips/VaswaniSPUJGKP17} with 6 encoder and 6 decoder layers for all datasets.
The number of attention heads, hidden size, and related hyperparameters are listed in Tab.~\ref{tab:hyperparams}.
The diffusion embedding dimension is 128.
For \scheduler, we set anchor points tuned to $20$-step sampling, setting $(\gamma_{\min},\gamma_{\max})=(0.90,0.95)$ and $(\lambda_{\min},\lambda_{\max})=(0.15,0.35)$.

\paragraph{Training.}
All models are trained with \texttt{fp16} on 4 NVIDIA H100 GPUs. We use an \texttt{inverse-sqrt} learning-rate schedule with 10{,}000 warmup steps and ${lr}_{\max}=5\times10^{-4}$ for all benchmarks. We set gradient norm clipping to 1.0, dropout to 0.1, and label smoothing to 0.1.
Runtime is approximately 8.5 hours for \textbf{WMT} and \textbf{Gigaword}, and around 4 hours on average for the other datasets.

\paragraph{Inference.}
At inference, the reverse process follows Eq.~\ref{eq:posterior}. Self-conditioning reuses the previous-step estimate, as in prior work. We apply Minimum Bayes-Risk (MBR) decoding~\cite{DBLP:conf/naacl/KumarB04} following DiffusionLM and DiffuSeq.

\paragraph{\noisescaling.}
We implement \noisescaling with three training phases and increase the scaling factor $\beta(n)$ over time.
The phase interval and scaling factor are given in Tab.~\ref{tab:hyperparams}. 
For example, on \textbf{WMT14} we use $\beta(n)=2.0$ for $n<100\text{K}$, $\beta(n)=3.0$ for $100\text{K}\le n<200\text{K}$, and $\beta(n)=4.0$ thereafter. This modification increases total training time by less than $5\%$ in our experiments.

\section{On The Effectiveness Of Minimum Bayesian Risk Decoding}
\label{sec:mbr_abl}

Since \framework builds on \texttt{Difformer}, we evaluate MBR decoding with both length beam search and noise-candidate search to improve output quality while maintaining diversity. We note that this search overhead can often be reduced by alternative stochastic length strategies~\cite{DBLP:conf/iclr/GongLF0K23,DBLP:conf/nips/WuFLZGS0LWGDC23}.
Here we focus on characterizing the search-space trade-offs under the standard MBR setup.

As shown in Tab.~\ref{tab:iwslt14_bleu}, increasing either the length beam or the noise beam improves \textit{BLEU}.
Fig.~\ref{fig:mbr_var} further determine which axis scales better.
The results reveal that scaling the length beam (vertical axis) yields faster gains than scaling the noise beam (horizontal axis).
Intuitively, a larger length beam mitigates length prediction errors and expands the set of plausible candidates, which in turn improves the effectiveness of MBR selection.
Noise beam, otherwise is not as effective, since our method already ensures convergence for each length.

\begin{figure*}[t!]
    \centering
    \includegraphics[width=0.75\linewidth]{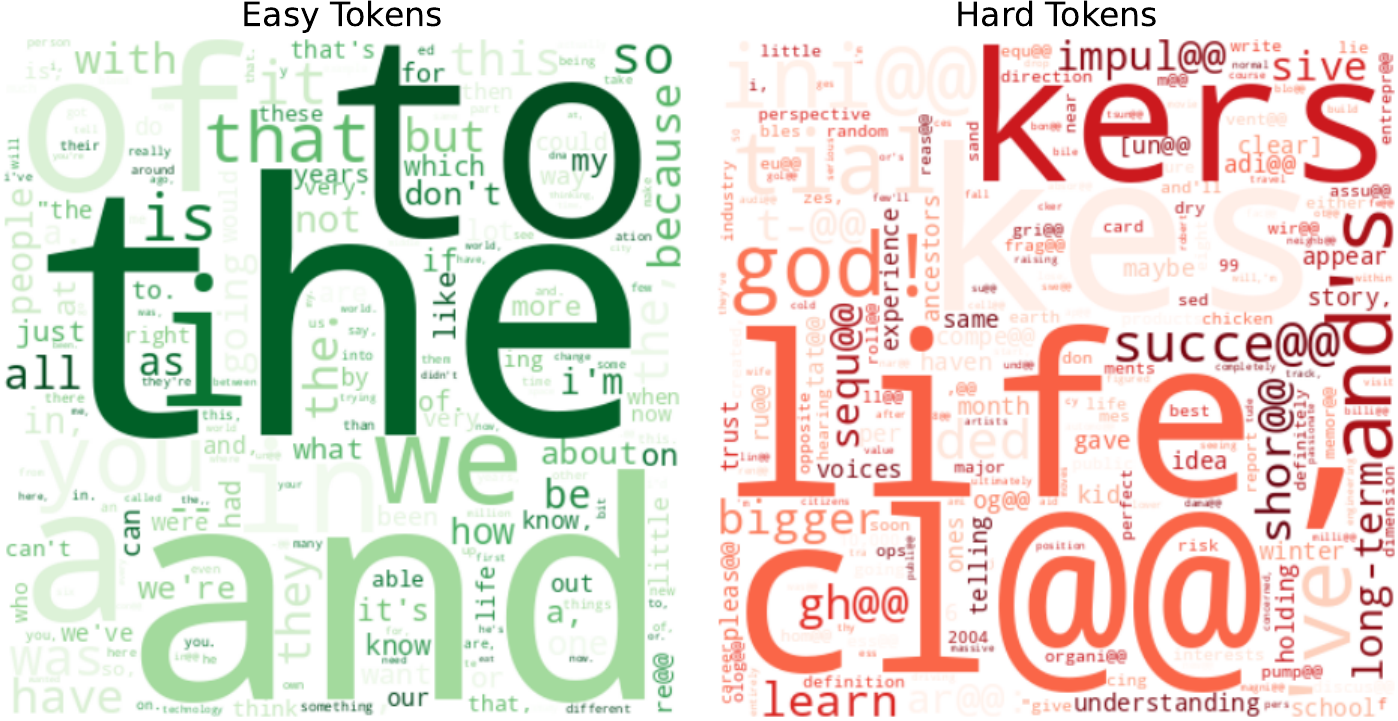}
    \caption{Word cloud of the easy and hard tokens during training.}
    \label{fig:word_cloud}
\end{figure*}

\begin{figure}
    \centering
    \includegraphics[width=0.9\linewidth]{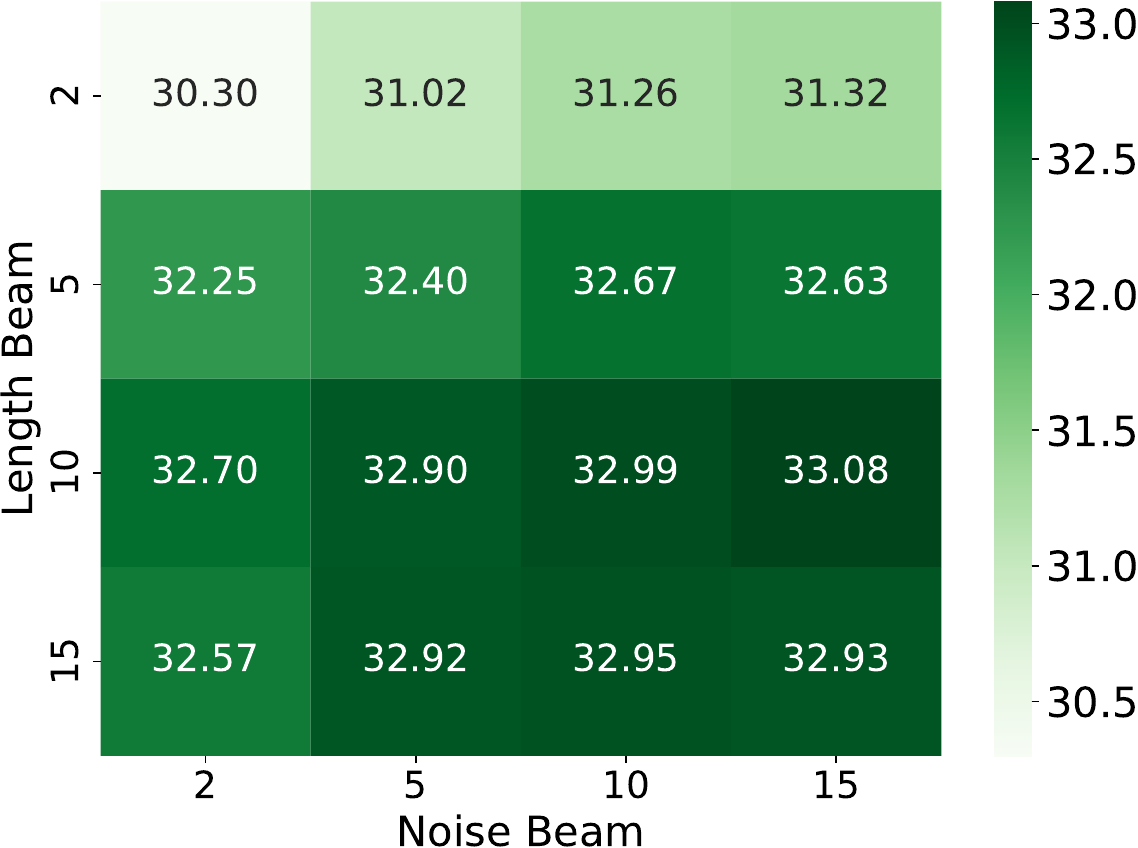}
    \caption{\textit{SacreBLEU} score on \textbf{IWSLT14} De-En with various length beams and noise beams.}
    \label{fig:mbr_var}
\end{figure}

\section{High And Low Confidence Tokens}
\label{sec:visualize}
Fig.~\ref{fig:word_cloud} visualizes tokens identified by \noisescaling as high or low confidence during training.
High-confidence tokens are predominantly common words, indicating that the model can already reconstruct them reliably. \noisescaling increases their effective noise level, forcing the model to denoise these ``easy'' tokens even under high noise, strengthening the self-conditioning signal.

In contrast, low-confidence tokens tend to be rarer or domain-specific words that appear less often in the training data. 
Their noise level is left unchanged, allowing the model to learn them under a less aggressive corruption regime. 
This behavior aligns with a coarse-to-fine denoising process: common words are forced to generate faster, while rare words use reliable generated words to refine prediction.

\section{Qualitative Results}
\label{sec:qualitative_res}

We present qualitative case studies in Tabs.~\ref{tab:example} and~\ref{tab:example_long} to illustrate how \framework changes generation dynamics at the instance level.
In Tab.~\ref{tab:example}, the first example shows that the baseline fails to effectively use the reused estimate, leading to persistent artifacts across steps (e.g., the repetition ``Summer Summer'' is not corrected in the next update). In contrast, \framework repairs the error in subsequent steps, consistent with \scheduler improving the robustness of the model to erroneous self-conditioning under few-step discretization.
The second example highlights that \framework produces a sharper, more accurate reconstruction target, aligning with the role of \noisescaling in strengthening supervision on high-confidence tokens.

Tabs.~\ref{tab:example_prefix},~\ref{tab:example_lexical}, and~\ref{tab:example_negative} further evaluate \scheduler adaptability to \texttt{Plaid}. Across these settings, \scheduler yields more coherent and consistent generations, reducing drift and improving paragraph-level continuity under limited denoising steps.

\begin{table*}[ht]
\small
\centering
\renewcommand{\arraystretch}{1.2}

\begin{tabular}{lcl}
\toprule
\midrule
& \textbf{Step} & \multicolumn{1}{c}{\textbf{Example 1}} \\
\midrule
\textbf{Source} & - & He also twice participated in the  Summer Olympics, starting in 1996.  \\
\textbf{Target} & - & He was also in the 1996 Summer Olympics. \\
\midrule
\multirow{2}{*}{\textbf{Baseline}} & 2 & He also twice in the Summer \textit{\textbf{Summer}} Olympics 1996 Summer 1996. \\
& 1 & He also twice in the Summer \textit{\textbf{Summer}} in the in 1996. \\
\midrule
\multirow{2}{*}{\textbf{\framework}} & 2 & He also twice twice in \textit{\textbf{in}} Summer \textit{\textbf{Summer}} Olympics, starting in 1996. \\
& 1 & He also twice twice  in the Summer Olympics, starting in 1996. \\
\midrule
\midrule
& \textbf{Step} & \multicolumn{1}{c}{\textbf{Example 2}} \\
\midrule
\textbf{Source} & - & Whedon served as an executive producer, along with Tim Minear. \\
\textbf{Target} & - & Whedon was the executive producer, along with Tim Minear. \\
\midrule
\multirow{2}{*}{\textbf{Baseline}} & 2 & Whedon was \textit{\textbf{an}} an executive producer with Tim Minear. \\
& 1 & He was \textit{\textbf{an}} an executive producer  with Tim Minear. \\
\midrule
\multirow{2}{*}{\textbf{\framework}} & 2 & Whedon was an executive producer with Tim Minear. \\
& 1 & He was an executive producer with Tim Minear. \\
\bottomrule
\end{tabular}
\caption{Examples of generation throughout the denoising process, with $\text{NFE}=2$.}
\label{tab:example}
\end{table*}

\begin{table*}[ht]
\centering
\small
\renewcommand{\arraystretch}{1.2}

\begin{tabular}{l l}
\toprule
\midrule
& \multicolumn{1}{c}{\textbf{Example 1}} \\
\midrule
\textbf{Source} & How do you think Federer, Nadal, Djokovic, Murray, Wawrinka, Delpo \& Cilic rank in terms of potential \\
& when it comes to greatest spot in the tennis?  \\
\textbf{Target} & How would you rank Federer, Nadal, Djokovic, Murray, Wawrinka, Delpo and Cilic in terms of greatness \\
& achieved and potential spot in the history of sport? \\
\textbf{Baseline} & How do you think about \textit{\textbf{the Federer, rank in a good of Mury Wawrinka}}, Delgreatest rank and Cilic rank in \\
& terms of for tennis? \\
\textbf{\framework} & What is the review of vic, MurAur, Murray, Dray, Wawrinka, DelPO \& Cilic rank in tennis? \\
\midrule
\midrule
& \multicolumn{1}{c}{\textbf{Example 2}} \\
\midrule
\textbf{Source} & How will the ban on 500 and 1000 rupee notes bring out the black money of the big shots who have lots \\
& of it in the Swiss bank in a different currency?  \\
\textbf{Target} & How will the decision to illegalize the 500 and 1000 Rs notes help to get rid of black money in the Swiss \\
& bank or maybe in other foreign banks and currencies? \\
\textbf{Baseline} & What will be the effect of banning 500 and 1000 notes on people \textit{\textbf{having money}} in Swiss banks? \\
\textbf{\framework} &  How is demonetizing the 500 and 1000 rupees notes affect Indian economy? How it affect the black money \\
& in Swiss bank accounts? \\
\bottomrule
\end{tabular}
\caption{Examples of generation on \textbf{Wiki-Auto} dataset. The resulting outputs demonstrate higher coherence in the long-term context.}
\label{tab:example_long}
\end{table*}

\begin{table*}[ht]
\centering
\small
\renewcommand{\arraystretch}{1.2}
\begin{tabularx}{\textwidth}{lX}
\hline
\hline
& \multicolumn{1}{c}{\textbf{Example 1}} \\
\hline
\textbf{Source} & This easy chicken curry recipe is made with just a handful of ingredients  \\
\textbf{\texttt{Plaid}} & This easy chicken curry recipe is made with just a handful of ingredients make provide add that helps easy add help give add that helps welcome \textcolor{red}{give help offer help for how you learn give help help help for act give help help help and perform give help help help for for how pay for help help help and even add help help help for give help that add give help up for the add provide help provide for help help help and add for how pay for help give help provide help and add give help help help and add help provide help to create help give help help help for for how doit add this give help help help that add help help give help help for and that give help add help give help give help for the run in the middle of the day add help help provide help help that gives simply give help help give help help provide help help provide help help for pay help help give help help pay give help give help give help help pay help the member on your family make help shop it} ... \\
\textbf{\texttt{Plaid} + \scheduler} & This easy chicken curry recipe is made with just a handful of ingredients at your disposal including a bow for glowing red ropes for strutter. It alone will be the trick since making anything else would normally be a lengthy method. Once that completes the dungeon is start laid out again towards the end where you learn the way to make naughty but once you find this little hidden dungeon point you will also need the basics of the tools to help you find your way to the recipe for cider seeds.$\hookleftarrow\hookleftarrow$So many would say these simple caves and dungeons are like paster glue in a Oddband environment. It is not just the recipe that is perhaps also the highlight in that the have side rooms for dungeon for completion in small spells or just locked in a spooky cave somewhere. The lack of questline support will prevent really good and fun faciettes. The formula the ingredients are a lot worth then anything else but for an hour of stats capable of requiring three seemed some how effective when you have the basics ... \\
\hline
\hline
& \multicolumn{1}{c}{\textbf{Example 2}} \\
\hline
\textbf{\textbf{Source}} & Generative models of text are very versatile: they can be used \\
\textbf{\texttt{Plaid}} & \textcolor{red}{Negative} models of text are very versatile: they can be used in the workplace and in the secondary market \textcolor{red}{provide availability provides published versions varies and competitive paid provided available for tender available as the selection of availability choices were very common on a high ... fair to read the initial sales framework very often ... provide selection for the except immediately in commonly based on how pricing control edited ask that everyone creates as a matter of how much published provide is actually apply control on that immediately on offered provide on exchange for them provide paid write line of commissioned provide access made applied to more creative offered works cover different parts of the market agree that paid sharing is much more what use sharing work on the ability to in understanding in support of a contribution of a meaningful piece on say you can provide featured provide paid paid assignment paid submissions management provided applied arrangements with mixed paid paid shared fit easily perform create commissioned work in private ... why is not offered applied availability on the subject or backing and reuse ... to ... more successful provide writing for working provide featured work submission creative exclusively created collection of icons is a symbol medium always learn how to make use of owns offering portfolio offering total free commissioned formats provided published licensing paid} \\
\textbf{\texttt{Plaid}+\scheduler} & Generative models of text are very versatile: they can be used to extract depth and status information cheaply. Others would argue that static models of text are a powerful aid for explersion; however, it was in the core caricatures that it didn’t offer much benefit because in the instance they were not comprehend, the translation to say showed a very high degree of complexity. If you are looking for help overcoming the CPP constraints, you will want to consider means with which the manual models have been effectively replaced. How much structure really can be leveraged here is left as an observation regarding numeric computation and the need to adapt.$\hookleftarrow\hookleftarrow$Even if the topic had not made its true contribution in the Enchanted Mongoose core system, the switch to primary structures would be significant. When I first saw the stream, I was told that the intent was to take the existing core system and replace it entirely with a trim and circular setup. That was described very well and perhaps would have been an option but as inconsistent with the extra-mainruction requirement as new core support for and directly related to static features. The details the stream provided suggest that they are not based on current text models, which probably means that the team is hoping for a design fix as you no longer get to the point of accurate input even in corramures. Primary layers, on the other hand, still provide a significant degree of input. You will have to navigate over simulated character emulation with replacement and recreate the basic static dynamic of the characters at a much condensed level. However, we are told that it will still be possible to use a simple filter in the secondary system to generate specific surface plans. \\
\hline
\end{tabularx}
\caption{Examples of 32-step long sequence generation on \textbf{Prefix Completion} with \texttt{Plaid-1B}. The output results indicate that \scheduler helps maintain the coherence and fluency of the paragraph, while ensuring consistency across sentences.}
\label{tab:example_prefix}
\end{table*}

\begin{table*}[ht]
\centering
\small
\renewcommand{\arraystretch}{1.2}
\begin{tabularx}{\textwidth}{lX}
\hline
\hline
& \multicolumn{1}{c}{\textbf{Example}} \\
\hline
\textbf{Source} & Lexically Constraints: Donald \\
\textbf{\texttt{Plaid}} & Sample 1: ... focused in comments on newly uncovered FBI collusion with President \textbf{Donald} Trump, Bill Nunes revelation that almost no trump prosecute matters, despite widespread evidence, making him the primary target of these very crimes ... \\
& Sample 2: ... It also includes some mud slinging with Obama FGA Chairman John Masters attacking \textbf{Donald Trump} for loose familiarity with the banking cartel. It’s no secret that Trump was beaten to the drapes by President Barack Obama on the actual exam. \\
& Sample 3: ... March in London Members of Parliament looking at paintings of Prime Minister \textbf{Donald} trump during EU festival day Getty Images Brexit protest ... \\
\textbf{\texttt{Plaid}+\scheduler} & Sample 1: ... It is also an equal, far more debasive sellout than the academia; some may have reasonably attributed the whole spread to \textbf{Donald} Nhutbout not clicking into the comments I have written this time; it seems to be unfortunate because both the coverage and commentary involved that continued focus could just seem to call out bias and denial ... \\
& Sample 2: ...  But at least on campuses, the so called feminist war on radical conservatives may actually be having a negative effect on Loomer attendance. When it comes to that \textbf{Donald} Trump presidency, we might actually find men who don’t have the viciousness, sorry love, of the feminist left ... \\
& Sample 3: ... The overproduction chells fields and destroys towns, creates great energy exhaustion and destabilization and allot in the free residents and the general approach to produce any other good or industry fails to expand operations and flows.$\hookleftarrow\hookleftarrow$\textbf{Donald} Stone: Within one of the smaller areas, on the street of \textbf{Donald} Stone’s home, lies the Exit Man, now the best live erifot and Pisa \_example courtesy of writer B. JonKing ... \\
\hline
\end{tabularx}
\caption{Examples of 32-step long sequence generation on \textbf{Lexically Constrained Generation} with \texttt{Plaid-1B}. The \texttt{Plaid-1B} samples lack diversity while ours cover more word combinations.}
\label{tab:example_lexical}
\end{table*}

\begin{table*}[ht]
\centering
\small
\renewcommand{\arraystretch}{1.2}
\begin{tabularx}{\textwidth}{lX}
\hline
\hline
& \multicolumn{1}{c}{\textbf{Example}} \\
\hline
\textbf{Source} & Negation: Donald but not Trump  \\
\textbf{\texttt{Plaid}} & ... [profile] Frederic Blanchfield Britain’s earliest basalt monument in Little Choyside, Chepole, Dorset. Source: MORI.mauipomwork-8x17-1.jpg... [profile] Marco Bori Special Projects Marco Bastoni Participants in discussion with Russian crionic residents in Venice. Artist/21-ponoco-8x17-1.jpg... [profile] Marco Baba.alexnet-8x17-1.jpg... [profile] Giananna Boone Interior Photography Histature from Rome Client – Social Architecture – Vanguard.photo/46-photograph-164x72-r.jpg... [profile] \textbf{Donald} Buckingham Industrial Art Izomage ProKnakon Call to exercise Tradarpones in Real Materials-stone-piping-8x17-608px.jpg... [slide] artmagazine-20x26x24.jpg... [slide] ... \\
\textbf{\texttt{Plaid}+\scheduler} & ... It was dress for contempt. The dishonesty of this whole development was clear, an astonishing and obvious, necessary to humanity, but seen in the remark Il Quartilver made which showed an almost impervolent fit of irony, necessary to discern truth from record. ‘ Our exile and our exile have been secured. The ocean has risen. Hell, \textbf{Donald} haath been a fool, before he came to save me, Taracharvi, and finally ... not only in these troubled times it was taking a giant it-own for their charisma, but sol gallur phir bariet die Pope \textbf{Donald}. If they couldn’t say ‘wrong’, where did they get it? The slanders at Timon knew to just make radios, they could have solved ... \\
\hline
\end{tabularx}
\caption{Examples of 32-step long sequence generation on \textbf{Negation} with \texttt{Plaid-1B}. The results demonstrate that the \texttt{Plaid-1B} sample is likely memorized from the dataset when against an unusual pattern, while our method produces a normal paragraph that satisfies the constraint.}
\label{tab:example_negative}
\end{table*}

\end{document}